\documentclass[sigconf,nonacm]{acmart}

\usepackage{bbm}
\usepackage{graphicx}
\usepackage{array}
\usepackage{amsfonts}

\usepackage{colortbl}
\usepackage[table,xcdraw,HTML]{xcolor}
\usepackage{subcaption}  
\usepackage{placeins}    
\usepackage{enumitem}
\usepackage{wrapfig}

\usepackage{booktabs,multirow}
\usepackage[ruled,vlined,linesnumbered]{algorithm2e}
\usepackage{adjustbox}

\DontPrintSemicolon
\SetKwFunction{BuildPrompt}{BuildPrompt}
\SetKwFunction{DualConstraintEval}{DualConstraintEval}
\SetKwFunction{PiTPOUpdate}{PiT-PO\_Update}
\SetKwFunction{SelectAndReset}{SelectAndReset}

\AtBeginDocument{%
  }

\acmConference[]{}{}

\begin{document}

\title{LLM-Based Scientific Equation Discovery via Physics-Informed Token-Regularized Policy Optimization}

\author{Boxiao Wang}
\email{wangboxiao22@mails.ucas.ac.cn}
\orcid{0009-0008-2970-7575}
\affiliation{%
  \institution{Institute of Automation, Chinese Academy of Sciences}
  \city{Beijing}
  \country{China}
}

\author{Kai Li}
\authornotemark[1]
\email{kai.li@ia.ac.cn}
    \affiliation{%
      \institution{Institute of Automation, Chinese Academy of Sciences}
      \city{Beijing}
      \country{China}
}

\author{Tianyi Liu}
\email{franktyliu@outlook.com}
    \affiliation{%
      \institution{State Key Laboratory of Aerodynamics}
      \city{Mianyang, Sichuan}
      \country{China}
}

\author{Chen Li}
\email{lichen@skla.cardc.cn}
    \affiliation{%
      \institution{State Key Laboratory of Aerodynamics}
      \city{Mianyang, Sichuan}
      \country{China}
}

\author{Junzhe Wang}
\email{wangjunzhe21@mails.ucas.ac.cn}
 \affiliation{%
   \institution{School of Mathematical Sciences,
University of Chinese Academy of Sciences}
   \city{Beijing}
   \country{China}
 }

\author{Yifan Zhang}
\authornotemark[1]
\email{yifan.zhang@ia.ac.cn}
    \affiliation{%
      \institution{Institute of Automation, Chinese Academy of Sciences}
      \city{Beijing}
      \country{China}
}

\author{Jian Cheng}
\email{jian.cheng@ia.ac.cn}
    \affiliation{%
      \institution{Institute of Automation, Chinese Academy of Sciences}
      \city{Beijing}
      \country{China}
}

\renewcommand{\shortauthors}{}

% \end{abstract}

\begin{abstract}
Symbolic regression aims to distill mathematical equations from observational data. Recent approaches have successfully leveraged Large Language Models (LLMs) to generate equation hypotheses, capitalizing on their vast pre-trained scientific priors. However, existing frameworks predominantly treat the LLM as a static generator, relying on prompt-level guidance to steer exploration. This paradigm fails to update the model’s internal representations based on search feedback, often yielding physically inconsistent or mathematically redundant expressions. In this work, we propose \textbf{PiT-PO} (Physics-informed Token-regularized Policy Optimization), a unified framework that evolves the LLM into an adaptive generator via reinforcement learning. Central to PiT-PO is a \textit{dual-constraint mechanism} that rigorously enforces hierarchical physical validity while simultaneously applying fine-grained, token-level penalties to suppress redundant structures. Consequently, PiT-PO aligns LLM to produce equations that are both scientifically consistent and structurally parsimonious. Empirically, PiT-PO achieves state-of-the-art performance on standard benchmarks and successfully discovers novel turbulence models for challenging fluid dynamics problems. We also demonstrate that PiT-PO empowers small-scale models to outperform closed-source giants, democratizing access to high-performance scientific discovery.

\end{abstract}

\maketitle

\section{Introduction}
Symbolic Regression (SR)~\cite{article} stands as a cornerstone of data-driven scientific discovery, uniquely capable of distilling interpretable mathematical equations from observational data. Unlike black-box models that prioritize mere prediction, SR elucidates the fundamental mechanisms governing system behavior, proving instrumental in uncovering physical laws~\cite{10.1093pnasnexuspgae467,10.1007978-3-031-29573-7_3}, modeling chemical kinetics~\cite{ICLR2025_a76b693f,DENG2023109010}, and analyzing complex biological dynamics~\cite{Wahlquist2024,shi2024alphaforgeframeworkdynamicallycombine}. 

However, the search for exact governing equations represents a formidable challenge, formally classified as an NP-hard problem~\cite{virgolin2022symbolicregressionnphard}. To navigate this vast search space, algorithmic strategies have evolved from Genetic Programming (GP)~\cite{doi:10.1126science.1165893,cranmer2023interpretablemachinelearningscience} and Reinforcement Learning (RL)~\cite{petersen2021deepsymbolicregressionrecovering} to Transformer-based architectures that map numerical data directly to symbolic equations~\cite{biggio2021neuralsymbolicregressionscales,kamienny2022endtoendsymbolicregressiontransformers,zhang2025ragsr}. Most recently, the advent of Large Language Models (LLMs) has introduced a new paradigm. Methods such as LLM-SR~\cite{shojaee2025llmsrscientificequationdiscovery} and LaSR~\cite{grayeli2024symbolicregressionlearnedconcept} leverage the pre-trained scientific priors and in-context learning capabilities of LLMs to generate equation hypotheses. These approaches typically employ an evolutionary search paradigm, where candidates are evaluated, and high-performing solutions are fed back via prompt-level conditioning to steer subsequent generation.

Despite encouraging progress, existing LLM-based SR methods remain constrained by severe limitations. First, most approaches treat the LLM as a static generator, relying primarily on prompt-level, \textit{verbal} guidance to steer the evolutionary search~\cite{shojaee2025llmsrscientificequationdiscovery,grayeli2024symbolicregressionlearnedconcept}. This ``frozen'' paradigm inherently neglects the opportunity to adapt and enhance the generative capability of the LLM itself based on evaluation signals, preventing the model from internalizing feedback and adjusting its generation strategies to the specific problem. Second, they typically operate in a physics-agnostic manner, prioritizing syntactic correctness over physical validity~\cite{shojaee2025llmsrscientificequationdiscovery,grayeli2024symbolicregressionlearnedconcept}. Without rigorous constraints, LLMs often generate equations that fit the data numerically but violate fundamental physical principles, rendering them prone to overfitting and practically unusable.

In this work, we propose to fundamentally shift the role of the LLM in SR from a static proposer to an adaptive generator.
We establish a dynamic feedback loop in which evolutionary exploration and parametric learning reinforce each other: evolutionary search uncovers diverse candidate equations and generates informative evaluation signals, while parametric adaptation enables the LLM to consolidate effective symbolic patterns and guide subsequent exploration more efficiently.
By employing \textit{in-search} fine-tuning, i.e., updating the LLM parameters during the evolutionary search process, we move beyond purely verbal, prompt-level guidance and introduce \textit{numerical} guidance that allows feedback to be directly internalized into the model parameters, progressively aligning LLM with the intrinsic properties of the target system.

To realize this vision, we introduce PiT-PO (Physics-informed Token-regularized Policy Optimization), a unified framework that bridges LLM-driven evolutionary exploration with rigorous verification.
PiT-PO is built upon two technical components. \textbf{1) In-Search LLM Evolution.} We implement the numerical guidance via reinforcement learning, which efficiently updates the LLM's parameters during the search process. Instead of relying on static pre-trained knowledge, this in-search policy optimization enables LLM to dynamically align its generative distribution with the structural characteristics of the specific task, effectively transforming general scientific priors into domain-specific expertise on the fly. \textbf{2) Dual Constraints as Search Guidance.} We enforce hierarchical physical constraints to ensure scientific validity, and uniquely, we incorporate a fine-grained regularization based on our proposed \textit{Support Exclusion Theorem}. This theorem allows us to identify mathematically redundant terms and translate them into token-level penalties, effectively pruning the search space and guiding LLM toward physically meaningful and structurally parsimonious equations.

Comprehensive experimental results demonstrate that PiT-PO achieves state-of-the-art performance across standard SR benchmarks, including the LLM-SR Suite~\cite{shojaee2025llmsrscientificequationdiscovery} and LLM-SRBench~\cite{shojaee2025llmsrbenchnewbenchmarkscientific}, and recovers the largest number of ground-truth equations among all evaluated methods. In addition to synthetic benchmarks, the effectiveness of PiT-PO is validated on an application-driven turbulence modeling task involving flow over periodic hills. PiT-PO improves upon traditional Reynolds-Averaged Navier-Stokes (RANS) approaches by producing anisotropic Reynolds stresses closer to Direct Numerical Simulation (DNS) references. The learned method shows enhanced physical consistency, with reduced non-physical extremes and better flow field predictions. Finally, PiT-PO maintains robust performance even when using resource-constrained small models, such as a quantized version of Llama-8B, and remains efficient under strict wall-clock time budgets, establishing a practical methodology for automated scientific discovery.

\begin{figure}[t]
  \centering
  
  \includegraphics[width=1\linewidth]{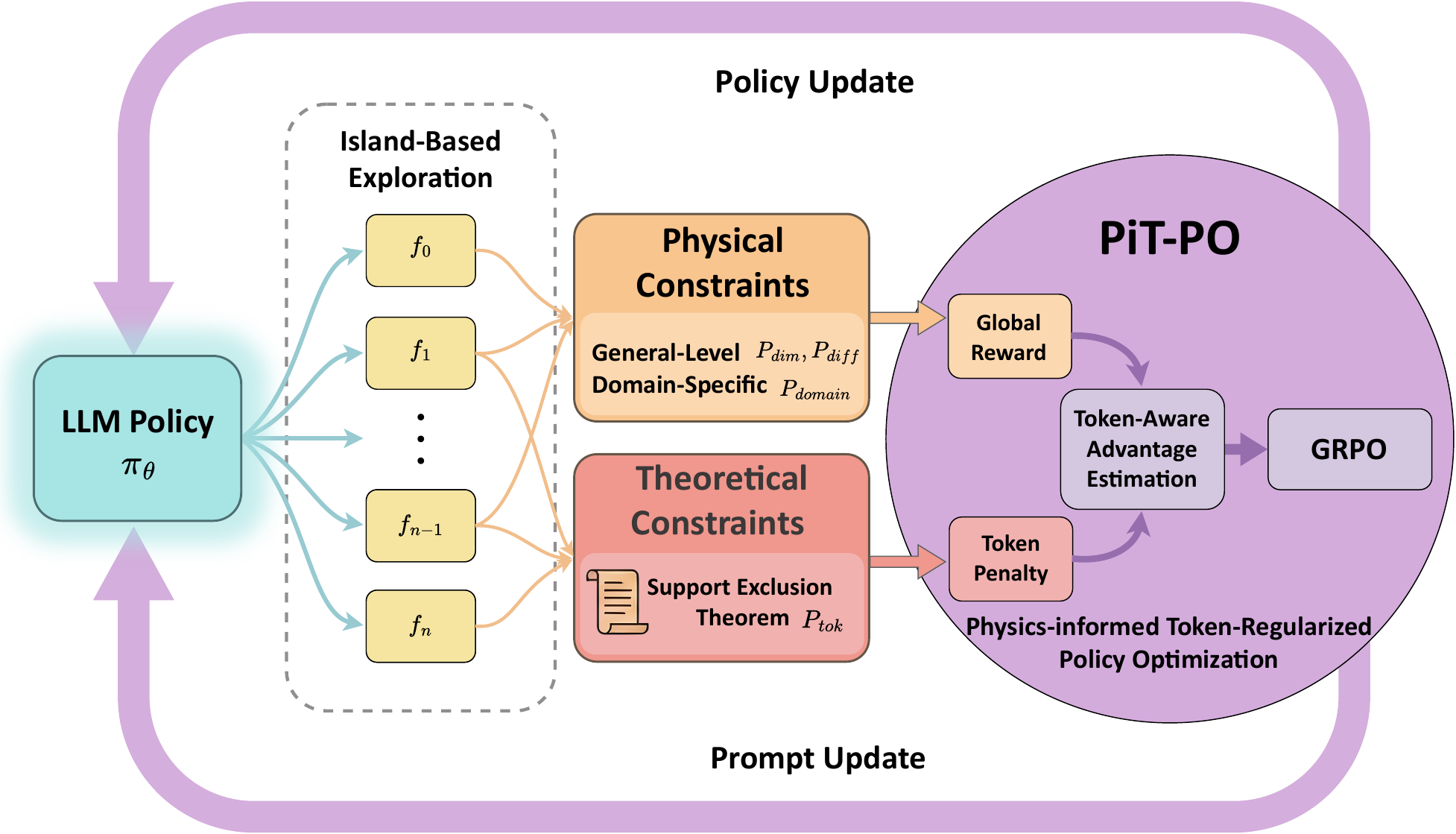}
  \caption{\textbf{The overall framework of PiT-PO.}
  PiT-PO transforms the LLM from a static proposer into an adaptive generator via a closed-loop evolutionary process. 
  The framework integrates dual-constraint evaluation—comprising physical constraints and theoretical constraints—to generate fine-grained token-level learning signals. 
  These signals guide the LLM policy update via reinforcement learning, ensuring the discovery of parsimonious, physically consistent equations.
  }
  
  \label{fig:framework}
\end{figure}

\section{Preliminaries}
\subsection{Problem Setup}
In SR, a dataset of input--output observations is given:
\begin{equation}
    D=\{(x_i,y_i)\}_{i=1}^{n},\; x_i\in\mathbb{R}^d, \; y_i\in\mathbb{R}.
\end{equation}
The objective is to identify a compact and interpretable function $f\in\mathcal{F}$ such that $f(x_i)\approx y_i$ for the observed samples, while retaining the ability to generalize to unseen inputs.

\subsection{LLM-based SR Methods}
Contemporary LLM-based approaches reformulate SR as a iterative program synthesis task. In this paradigm, typified by frameworks such as LLM-SR~\cite{shojaee2025llmsrscientificequationdiscovery}, the discovery process is decoupled into two phases: structure proposal and parameter estimation. Specifically, the LLM functions as a symbolic generator, emitting functional skeletons with placeholders for learnable coefficients. A numerical optimizer (e.g., BFGS~\cite{fletcher1987_practical_methods}) subsequently fits these constants to the observed data. To navigate the combinatorial search space, these methods employ an evolutionary-style feedback loop: high-fitness equations are maintained in a pool to serve as in-context examples, prompting the LLM to refine subsequent generations. Our work leverages this architecture as a backbone, but fundamentally redefines the LLM's role from a static proposer to an adaptive generator.

\subsection{Group Relative Policy Optimization}
Group Relative Policy Optimization (GRPO)~\cite{shao2024deepseekmathpushinglimitsmathematical} is a RL algorithm tailored for optimizing LLMs on reasoning tasks, characterized by its efficient baseline estimation without the need for a separate value network.
In the context of LLM-based SR, the generation process is modeled as a Markov Decision Process (MDP), where LLM functions as a policy $\pi_\theta$ that generates a sequence of tokens $o = (t_1, \dots, t_L)$ given a prompt $q$.
For each $q$, GRPO samples a group of $G$ outputs $\{o_1, \dots, o_G\}$ from the sampling policy $\pi_{\theta_{old}}$. GRPO maximizes the following surrogate loss function:
\begin{equation}
\footnotesize
\mathcal{J}_{GRPO}(\theta) = \mathbb{E}_{q \sim P(Q), \{o_i\} \sim \pi_{\theta_{old}}} \left[ \frac{1}{G} \sum_{i=1}^G \left( \frac{1}{L_i} \sum_{k=1}^{L_i} \mathcal{L}^{clip}_{i,k}(\theta) - \beta \mathbb{D}_{KL}(\pi_\theta || \pi_{ref}) \right) \right],
\end{equation}
where $\pi_{ref}$ is the reference policy to prevent excessive deviation, and $\beta$ controls the KL-divergence penalty. The clipping term $\mathcal{L}^{clip}_{i,k}(\theta)$ ensures trust region updates:
\begin{equation}
\footnotesize
\mathcal{L}^{clip}_{i,k}(\theta) = \min \left( \frac{\pi_\theta(t_{i,k}|q, o_{i,<k})}{\pi_{\theta_{old}}(t_{i,k}|q, o_{i,<k})} \hat{A}_i, \;\text{clip}\left(\frac{\pi_\theta(t_{i,k}|q, o_{i,<k})}{\pi_{\theta_{old}}(t_{i,k}|q, o_{i,<k})}, 1-\epsilon, 1+\epsilon\right) \hat{A}_i \right).
\end{equation}
Here, $\epsilon$ is the clipping coefficient. 
A distinctive feature of GRPO lies in its advantage estimation, it computes the advantage $\hat{A}_i$ by standardizing the reward $R(o_i)$ relative to the group:
\begin{equation}
\hat{A}_i = \frac{R(o_i) - \text{mean}(\{R(o_j)\})}{\text{std}(\{R(o_j)\})}.
\end{equation}
Consequently, every token within the sequence $o_i$ is assigned the exact same feedback signal. This coarse granularity treats valid and redundant terms indistinguishably, a limitation that our work addresses by introducing fine-grained, token-level regularization.

\section{Method}

We propose PiT-PO (Physics-informed Token-regularized Policy Optimization), a framework that evolves LLM into an adaptive, physics-aware generator. PiT-PO establishes a closed-loop evolutionary process driven by two synergistic mechanisms:
(1) a dual-constraint evaluation system that rigorously assesses candidates through hierarchical physical verification and theorem-guided redundancy pruning; and (2) a novel policy optimization strategy that updates LLM using fine-grained, token-level feedback derived from these constraints. This combination effectively aligns the LLM with the intrinsic structure of the problem, guiding the LLM toward solutions that are not only numerically accurate but also structurally parsimonious and scientifically consistent.

\subsection{Dual-Constraint Learning Signals}
\label{dual-constrint}
Navigating the combinatorial space of symbolic equations requires rigorous guidance. We employ a reward system driven by dual constraints: physical constraints delineate the scientifically valid region, while theoretical constraints drive the search toward simpler equations by identifying and pruning redundant terms.

\subsubsection{Hierarchical Physical Constraints.}

To ensure scientific validity, we construct a hierarchical filter that categorizes constraints into two levels: general properties and domain-specific priors.

\textbf{General-Level Constraints.} We enforce fundamental physical properties applicable across scientific disciplines. To prune physically impossible structures (e.g., adding terms with mismatched units), we assign penalty-based rewards for Dimensional Homogeneity ($P_{dim}$) and Differentiability ($P_{diff}$). The former strictly penalizes equations with unit inconsistencies, while the latter enforces smoothness on the data-defined domain.

\textbf{Domain-Specific Constraints.} To tackle specialized tasks, we inject expert knowledge as inductive biases. We define the domain-specific penalty $P^{(j)}_{domain}$ to penalize candidate equations that violate the $j$-th domain-specific constraint. Taking the turbulence modeling task (detailed in~Appendix~\ref{subsec:turb_domain_constraints}) as a representative instantiation, we enforce four rigorous constraints: (1) Realizability~\cite{Pope2000}, ensuring the Reynolds stress tensor has positive eigenvalues; (2) Boundary Condition Consistency~\cite{Monkewitz2021}, requiring stresses to decay to zero at the wall; (3) Asymptotic Scaling~\cite{Tennekes1972, 0258-1825(2019)03-0419-07}, enforcing the cubic relationship between stress and wall distance in the viscous sublayer; and (4) Energy Consistency~\cite{Pope2000,MOCHIZUKI2000}, aligning predicted stress with turbulent kinetic energy production. 

This hierarchical design effectively embeds physical consistency as a hard constraint in the reward function, prioritizing scientific validity over mere empirical fitting.

\subsubsection{Theorem-Guided Mathematical Constraints}
While physical constraints ensure validity, they do not prevent mathematical redundancy.
To rigorously distinguish between essential terms and redundant artifacts, we introduce the Support Exclusion Theorem. 

Let $\mathcal{S}$ denote the \textbf{full support set} containing all candidate basis functions $\{\phi_j\}$. The \textbf{ground truth equation} is $f^* = \sum_{j \in \mathcal{S}'} a_j \phi_j$, where $\mathcal{S}' \subseteq \mathcal{S}$ is the \textbf{true support set} (i.e., the indices of basis functions that truly appear in the governing equation), and $\{a_j\}_{j \in \mathcal{S}'}$ are the corresponding true coefficients. Consider a candidate equation $f = \sum_{j \in \mathcal{K}} b_j \phi_j$, where $\mathcal{K} \subseteq \mathcal{S}$ represents the \textbf{current support set} (i.e., the selected terms in the skeleton), $\mathbf{b} = \{b_j\}_{j \in \mathcal{K}}$ are the optimized coefficients derived from the data. We define the empirical Gram matrix of these basis functions as $G \in \mathbb{R}^{|\mathcal{S}| \times |\mathcal{S}|}$ and the corresponding Projection Matrix as $T$, where $T_{ij} := G_{ji} / G_{ii}$.

\begin{theorem}[Support Exclusion Theorem]
\label{thm:support_exclusion}

Assume the ground-truth support is finite and satisfies $|\mathcal{S}'|\le M$, and let the true function coefficients be bounded by $A \le |a_j| \le B$ for all $j \in \mathcal{S}'$. A term $\phi_i$ ($i \in \mathcal{K}$) is theoretically guaranteed to be a false discovery (not in the true support $\mathcal{S}'$) if its fitted coefficient magnitude satisfies:
\begin{equation}
|b_i| < A - \left( \underbrace{\sum_{j \in \mathcal{K}, j \neq i} (B + |b_j|) |T_{ij}|}_{\text{Internal Interference}} +\underbrace{B \sum_{k=1}^{m} s_{(k)}}_{\text{External Interference}} \right).
\end{equation}
$s(k)$ denotes the $k$-th largest value in $\{|T_{i\ell}|:\ell\in\mathcal{S}\setminus\mathcal{K}\}$, and $m := \min\!\big(M-1,\;|\mathcal{S}\setminus\mathcal{K}|\big)$.

\end{theorem}

Detailed definitions of all notations and the rigorous proof of Theorem~\ref{thm:support_exclusion} are provided in Appendix~\ref{app:therom}. This theorem formalizes the intuition that coefficients of redundant terms (absent from the true support $\mathcal{S}'$) have significantly smaller magnitudes than those of valid components.

Specifically, after fitting $\mathbf{b}$, we compute the normalized coefficient ratio $\tau_i=|b_i|/(\sum_j |b_j|+ \epsilon)$. We introduce a threshold $\rho \in (0,1)$ to identify potentially redundant terms. Terms satisfying $\tau_i > \rho$ incur no penalty, while components with $\tau_i \le \rho$ are considered redundant.
To suppress these redundancies, we define a \textbf{token penalty} for each token in redundant term $i$:
\begin{equation}
\label{pik}
P_{tok} = p \cdot \max\left(0, -\log\left(|b_i| + \epsilon\right)\right),
\end{equation}
where $p > 0$ is a scaling coefficient.
We use a logarithmic scale to impose stronger penalties on terms with smaller coefficients.

By integrating this penalty into the policy optimization, we guide the LLM to reduce the probability of generating redundant terms, thereby steering the optimization toward parsimonious equations.

\subsection{Token-Aware Policy Update}

\label{3.2}
Our proposed PiT-PO effectively operationalize the hierarchical constraints and theoretical insights derived in Section~\ref{dual-constrint}. Unlike standard GRPO that assign a uniform scalar reward to the entire generated sequence, our method transitions the learning process from coarse-grained sequence scoring to fine-grained token-level credit assignment. This ensures that the policy not only learns to generate physically valid equations but also explicitly suppresses theoretically redundant terms.

\subsubsection{Global Reward with Gated Constraints}

The optimization is driven by a composite global reward, $R_{global}$, which balances fitting accuracy, structural parsimony, and physical consistency. Formally, for a sampled equation $o_i$, the rewards are defined as follows:
 
\textbf{Fitting Accuracy} ($R_{fit}$). We use the normalized log-MSE to encourage precise data fitting:
       \begin{equation}
        R_{fit} = -\alpha \log(\text{MSE} + \epsilon),
       \end{equation}
where $MSE = \frac{1}{n} \sum_{i=1}^n (y_i - \hat{y}_i)^2$.

\textbf{Complexity Penalty} ($P_{cplx}$). Adhering to Occam's Razor, we penalize structural complexity based on the Abstract Syntax Tree (AST)~\cite{10.11451083142.1083143} node count:
             \begin{equation}
              P_{cplx} = \lambda_{len} \cdot \text{Length}(\text{AST}).
              \end{equation}
In PiT-PO, each equation generated by the LLM is represented as a Python function and parsed into an AST, where each node corresponds to a variable or operator. The total node count provides a meaningful estimate of structural complexity.

\textbf{Gated Physical Penalty} ($P_{phy}$). Imposing strict physical constraints too early can hinder exploration, causing the model
to discard potentially promising functional forms. We therefore activate physical penalties only after the candidate equation reaches a baseline fitting accuracy threshold ($\delta_{gate}$). Specifically, we define
\begin{equation}
\resizebox{\columnwidth}{!}{$
P_{phy}(o_i)
= \mathbbm{1}\!\big({\scriptsize \mathrm{MSE}(o_i)} < \delta_{gate}\big)\,
\Big( P_{diff}(o_i) + P_{dim}(o_i) + \sum_{j=1}^{J} P^{(j)}_{\text{domain}}(o_i) \Big),
$}
\end{equation}
where $\mathbb{1}(\cdot)$ is the indicator function. This mechanism effectively creates a soft curriculum: it allows ``free'' exploration in the early stages and enforcing strict physical compliance only after the solution enters a plausible region.

The total reward $R_{global}$ is then formulated as:
\begin{equation}
R_{global}(o_i)= R_{fit}(o_i) - P_{cplx}(o_i) - P_{phy}(o_i).
\end{equation}

\subsubsection{Fine-Grained Advantage Estimation}
Standard GRPO applies a uniform advantage across all tokens in a sequence. 
We refine this by synthesizing the global reward with the token-level penalty $P_{tok}$ (Equation~\ref{pik}).
Specifically, we define the \textbf{token-aware advantage} $\hat{A}_{i,k}$ for the $k$-th token in the $i$-th sampled equation as:
\begin{equation}
\hat{A}_{i,k} = \underbrace{\frac{R_{global}(o_i) - \mu_{group}}{\sigma_{group}}}_{\text{Global Standardization}} - \underbrace{P_{i,k}}_{\text{Local Pruning}}.
\end{equation}

Here, the first term standardizes the global reward against the group statistics ($\mu_{group}, \sigma_{group}$), reinforcing equations that satisfy multi-objective criteria relative to their peers. 
The second term, $P_{i,k}$, applies a targeted penalty to suppress redundancy.
Specifically, we set $P_{i,k}=0$ if token $k$ belongs to a non-redundant term, and $P_{i,k}=P_{tok}$ otherwise. This ensures that penalties are applied exclusively to tokens contributing to mathematically redundant structures, while valid terms remain unaffected.

Substituting this token-aware advantage into the GRPO objective, the policy gradient update of our PiT-PO becomes:
\begin{equation}
\nabla \mathcal{J}_{PiT-PO} \propto \sum_{i,k} \hat{A}_{i,k} \nabla \log \pi_\theta(t_{i,k} | o_{i,<k}).
\end{equation}
This creates a dual-pressure optimization landscape: global rewards guide the policy toward physically consistent and accurate equations, while local penalties surgically excise redundant terms. This ensures the final output aligns with the sparse, underlying physical laws rather than merely overfitting numerical data.

{\setlength{\textfloatsep}{2pt}
\begin{algorithm}[t]
\caption{PiT-PO Overall Training Pipeline}
\label{alg:pitpo_overall_training}
\KwIn{Dataset $D=\{(x_i,y_i)\}_{i=1}^n$; LLM $\pi_{\theta}$; number of islands $N$; group size $G$; iterations $T$.}
\KwOut{Best equation $o^{*}$.}
\BlankLine

Initialize $o^{*}$, $s^{*}$, and buffers $\mathcal{B}_j \leftarrow \emptyset$ for $j=1,\dots,N$\;

\For{$t \leftarrow 1$ \KwTo $T$}{
  \tcp{Stage 1: Island-Based Exploration}
  \For{$j \leftarrow 1$ \KwTo $N$}{
    $q_j \leftarrow \textsc{BuildPrompt}(D,\mathcal{B}_j)$\tcp*[r]{in-context rule}
    $\{o_i\}_{i=1}^{G} \sim \pi_{\theta}(\cdot \mid q_j)$\tcp*[r]{sample a group}
    \For{$i \leftarrow 1$ \KwTo $G$}{
      $(R_i,\{P_{i,k}\}) \leftarrow \textsc{DualConstraintEval}(o_i,D)$\tcp*[r]{$R_i=R_{\mathrm{global}}(o_i)$}
      $\mathcal{B}_j \leftarrow \mathcal{B}_j \cup \{(q_j,o_i,R_i,\{P_{i,k}\})\}$ 
      
      \If{$R_i > s^{*}$}{
        $o^{*} \leftarrow o_i$;\quad $s^{*} \leftarrow R_i$
      }
    }
  }

  \tcp{Stage 2: In-Search LLM Evolution}
  $\theta \leftarrow \textsc{PiT-PO\_Update}(\theta,\{\mathcal{B}_j\}_{j=1}^{N},\pi_{{\theta}})$

  \tcp{Stage 3: Hierarchical Selection}
  $\{\mathcal{B}_j\}_{j=1}^{N} \leftarrow \textsc{SelectAndReset}(\{\mathcal{B}_j\}_{j=1}^{N})$
}
\Return{$o^{*}$}\;
\end{algorithm}
}

\subsection{Overall Training Pipeline}

We orchestrate the PiT-PO framework through a closed-loop evolutionary RL cycle. As illustrated in Algorithm~\ref{alg:pitpo_overall_training}, the training process iterates through three synergistic phases:

Phase 1: Island-Based Exploration (Data Generation). To prevent premature convergence to local optima, a common pitfall in SR, we employ a standard multi-island topology ($N$ islands) to structurally enforce search diversity~\cite{cranmer2023interpretablemachinelearningscience,romera2024mathematical}. 
Each island $j$ maintains an isolated experience buffer $\mathcal{B}_j$, evolving its own lineage of equations.
This information isolation allows distinct islands to cultivate diverse functional forms independently.

Phase 2: In-Search LLM Evolution (Policy Update). This phase transforms the collected data into parametric knowledge. We aggregate the trajectories from all $N$ islands into a global batch to perform policy optimization using PiT-PO. By minimizing the loss, the model explicitly lowers the probability of generating mathematically redundant tokens and physically inconsistent structures. To ensure computational efficiency during this iterative search, we implement the update using Low-Rank Adaptation (LoRA)~\cite{hu2021loralowrankadaptationlarge}.

Phase 3: Hierarchical Selection (Population Management). We apply standard survival-of-the-fittest mechanisms to maintain population quality. Local buffers $\mathcal{B}_j$ are updated by retaining only top-performing candidates, while underperforming islands are periodically reset with high-fitness seeds to escape local optima.

This cycle establishes a reciprocal reinforcement mechanism: the island-based exploration maintains search diversity, while policy update consolidates these findings into the model weights, progressively transforming the LLM into a domain-specialized scientific discoverer.

\section{Experiments}

\subsection{Setup}

\noindent\textbf{Benchmarks.}
To provide a comprehensive evaluation of PiT-PO, we adopt two widely used benchmarks to compare against state-of-the-art baselines:

\textit{LLM-SR Suite}~\cite{shojaee2025llmsrscientificequationdiscovery}. This suite comprises four tasks spanning multiple scientific domains: \textit{Oscillation 1 \& 2} (Nonlinear Oscillatory Systems) feature dissipative couplings and non-polynomial nonlinearities with explicit forcing, making recovery of the correct interaction terms from trajectory data non-trivial; \textit{E. coli Growth} ~\cite{annurev:contentjournals10.1146annurev.mi.03.100149.002103,Rosso1995}models multivariate population dynamics with strongly coupled, multiplicative effects from nutrients, temperature, and acidity; and \textit{Stress-Strain}~\cite{Aakash2019} uses experimental measurements of Aluminum 6061-T651 and exhibits temperature-dependent, piecewise non-linear deformation behavior. Detailed information about these tasks is provided in Appendix~\ref{appendix:llmsr-suite}.

\textit{LLM-SRBench:}~\cite{shojaee2025llmsrbenchnewbenchmarkscientific} To evaluate generalization beyond canonical forms, we adopt the comprehensive LLM-SRBench benchmark, which contains 239 tasks organized into two complementary subsets, LSR-Transform and LSR-Synth. LSR-Transform changes the prediction target to rewrite well-known physics equations into less common yet analytically equivalent forms, producing 111 transformed tasks. This design aims to reduce reliance on direct memorization of canonical templates and tests whether a method can recover the same physical law under non-trivial variable reparameterizations. Complementarily, LSR-Synth composes equations from both known scientific terms and synthetic but plausible terms to further assess discovery beyond memorized templates: candidate terms are proposed by an LLM under domain context, assembled into full equations, and then filtered through multiple checks, including numerical solvability, contextual novelty, and expert plausibility, yielding 128 synthetic tasks. Further details are given in Appendix~\ref{appendix:llmsrbench}.

\medskip

\noindent\textbf{Baselines.}

We compare {PiT-PO} against representative baselines spanning both classical and LLM-based SR methods. For the four tasks in the LLM-SR Suite, we include {GPlearn}, a genetic programming-based SR approach; {PySR}~\citep{grayeli2024symbolicregressionlearnedconcept}, which couples evolutionary search with symbolic simplification; {uDSR}~\citep{NEURIPS2022_dbca58f3}, which replaces the RNN policy in DSR with a pretrained Transformer and employs neural-guided decoding; {RAG-SR}~\citep{zhang2025ragsr}, which incorporates structure retrieval to assist equation generation; and {LLM-SR}~\citep{shojaee2025llmsrscientificequationdiscovery}. For the broader LLM-SRBench benchmark, we further compare against leading LLM-based SR methods, including {SGA}~\citep{pmlr-v235-ma24m}, which integrates LLM-driven hypothesis proposal with physics-informed parameter optimization in a bilevel search framework, and {LaSR}~\citep{grayeli2024symbolicregressionlearnedconcept}, which leverages abstract symbolic concepts distilled from prior equations to guide hybrid LLM--evolutionary generation.

\medskip
\noindent\textbf{Evaluation metrics.}
We evaluate methods using {Accuracy to Tolerance} and {Normalized Mean Squared Error (NMSE)}.
For a tolerance $\tau$, we report $\mathrm{Acc}_{\mathrm{all}}(\tau)$ ~\cite{biggio2021neuralsymbolicregressionscales}and $\mathrm{Acc}_{\mathrm{avg}}(\tau)$ based on relative error:
{\small $\mathrm{Acc}_{\mathrm{all}}(\tau)=\mathbbm{1}\!\bigl(
\max_{1\le i\le N_{\mathrm{test}}}\bigl|\tfrac{\hat{y}_i-y_i}{y_i}\bigr|
\le\tau\bigr)$
} and
{\small $\mathrm{Acc}_{\mathrm{avg}}(\tau)=\tfrac{1}{N_{\mathrm{test}}}\sum_{i=1}^{N_{\mathrm{test}}}
\mathbbm{1}\!\bigl(\bigl|\tfrac{\hat{y}_i-y_i}{y_i}\bigr|\le\tau\bigr)$},
where $\hat{y}_i$ and $y_i$ denote the predicted and ground-truth values at the $i$-th test point, respectively.
We additionally report $\mathrm{NMSE}=\frac{1}{N_{\mathrm{test}}}\sum_{i=1}^{N_{\mathrm{test}}}\frac{(\hat{y}_i-y_i)^2}{\mathrm{Var}(y)}$ to assess overall numerical accuracy. We additionally adopt the Symbolic Accuracy (SA) metric~\cite{shojaee2025llmsrbenchnewbenchmarkscientific}, which directly measures whether the discovered equation recovers the correct symbolic form (i.e., whether it is mathematically equivalent to the ground-truth equation up to fitted constants).

\medskip
\noindent\textbf{Hyperparameter Configurations.}
All experiments were run for 2,500 search iterations. To ensure a fair comparison, all hyperparameters related to LLM generation and search were kept consistent with the default configuration of LLM-SR. For the in-search policy optimization specific to PiT-PO, we use a learning rate of $1 \times 10^{-6}$, a group size of $G=4$, and a multi-island setting of $N=4$, resulting in an effective per-device batch size of $G \times N = 16$. The coefficient of the KL regularization term was set to 0.01, and the LoRA rank was set to $r=16$. In addition, experiments were conducted on a single NVIDIA RTX 3090 using 4-bit quantized Llama-3.2-1B-Instruct, Llama-3.2-3B-Instruct, and Llama-3.1-8B-Instruct~\cite{kassianik2025llama31foundationaisecurityllmbase8btechnicalreport} to evaluate the training stability and performance transferability of PiT-PO across different parameter scales under constrained compute and memory budgets. More details are in Appendix~\ref{app:pitpo_hparams}.

\begin{table*}[!t]
\centering
\normalsize
\resizebox{\textwidth}{!}{
\begin{tabular}{lcccccccc}
\hline
\multirow{2}{*}{\textbf{Models}} &
\multicolumn{2}{c}{\textbf{Oscillation 1}} &
\multicolumn{2}{c}{\textbf{Oscillation 2}} &
\multicolumn{2}{c}{\textbf{E. coli growth}} &
\multicolumn{2}{c}{\textbf{Stress-Strain}} \\
\cline{2-9}
& {Acc\textsubscript{avg-0.001}(\%)$\uparrow$} & {NMSE$\downarrow$}
& {Acc\textsubscript{avg-0.001}(\%)$\uparrow$} & {NMSE$\downarrow$}
& {Acc\textsubscript{avg-0.1}(\%)$\uparrow$} & {NMSE$\downarrow$}
& {Acc\textsubscript{avg-0.1}(\%)$\uparrow$} & {NMSE$\downarrow$} \\
\hline
GPlern & 0.11 & 0.0972 & 0.05 & 0.2000 & 0.76 & 1.0023 & 28.43 & 0.3496 \\
uDSR   & 1.78 & 0.0002 & 0.36 & 0.0856 & 1.12 & 0.5059 & 59.15 & 0.0639 \\
PySR   & 3.80 & 0.0003 & 7.02 & 0.0002 & 2.80 & 0.4068 & 70.60 & 0.0347 \\
RAG-SR & 39.47 & 1.49e-6 & 0.43 & 0.0282 & 2.04 & 0.2754 & 76.28 & 0.0282 \\

LLM-SR (Mixtral) & \cellcolor{gray!20}\textbf{100.00} & 1.32e-11 & 99.98 & 1.18e-11 & 2.88 & 0.0596 & 71.44 & 0.0276 \\
LLM-SR (4o-mini) & 99.92 & 8.84e-12 & 99.97 & 8.70e-10 & 5.52 & 0.0453 & \cellcolor{gray!20}\textbf{85.33} & 0.0245 \\

\hline

LLM-SR (Llama-3.1-8B)  & 59.58 & 1.17e{-6} & 99.96 & 9.66e-10 & 4.86 & 0.0555 & 77.74 & 0.0246 \\
LLM-SR (Llama-3.2-3B)  & 39.45  & 1.76e-6 & 66.34 & 6.99e{-7} & 1.08 & 0.3671 & 74.78 & 0.0324 \\
LLM-SR (Llama-3.2-1B)  & 3.28 & 4.47e-4 & 7.81 & 0.0002 & 1.70 & 0.5801 & 30.35 & 0.3801 \\

\hline

PiT-PO (Llama-3.1-8B) &
\cellcolor[HTML]{D9D9FF}\textbf{100.00} &
\cellcolor[HTML]{D9D9FF}\textbf{6.41e{-31}} &
\cellcolor{gray!20}\textbf{99.99} &
\cellcolor{gray!20}\textbf{2.11e{-13}} &
\cellcolor{gray!20}\textbf{10.42} &
\cellcolor{gray!20}\textbf{0.0090} &
84.45 &
\cellcolor{gray!20}\textbf{0.0136} \\

PiT-PO (Llama-3.2-3B) &
\cellcolor[HTML]{D9D9FF}\textbf{100.00} &
\cellcolor[HTML]{D9D9FF}\textbf{7.58e{-31}} &
99.97 & 9.77e{-10} & 7.01 & 0.0248 & 84.54 & 0.0156 \\

PiT-PO (Llama-3.2-1B) & 99.95 & 1.34e{-11} & 99.97 & 1.70e{-8} & 4.76 & 0.0240 & 76.91 & 0.1767 \\

\hline
\end{tabular}
}
\caption{Overall performance on LLM-SR Suite.}
\label{tab:main}

\end{table*}

\subsection{PiT-PO Demonstrates Superior Equation Discovery Capability}

As evidenced in Table~\ref{tab:main}, PiT-PO establishes a new state-of-the-art on LLM-SR Suite. It consistently dominates baseline methods across all metrics, achieving the highest accuracy while maintaining the lowest NMSE in nearly all test cases. Crucially, when controlling for the LLM backbone, PiT-PO yields a substantial performance margin over LLM-SR, validating the effectiveness of our in-search policy optimization framework. Notably, PiT-PO is the only approach to successfully identify the exact ground-truth equation for the Oscillator 1. This structural precision extends to the larger-scale LLM-SRBench (Table~\ref{tab:tb2}), where PiT-PO achieves the highest symbolic accuracy across all categories. These results collectively demonstrate that PiT-PO not only fits data numerically but excels in uncovering the true underlying equations.

These quantitative gains are not accidental but stem from PiT-PO's structural awareness. 
Analysis of the iterative trajectories of LLM-SR and PiT-PO in Appendix~\ref{app:eq_case_study} corroborates this conclusion: the iterations of LLM-SR remain persistently influenced by clearly incorrect terms, which violate physical meaning despite providing strong numerical fits, as well as by additional nuisance terms. Consequently, the search of LLM-SR often stagnates in a low-MSE regime without reaching the correct structure. In contrast, once PiT-PO enters the same regime, the dual constraints rapidly eliminate terms that improve fitting performance but are structurally incorrect. This behavior highlights the central advantage of the proposed dual-constraint learning signals, in which the physical constraints and token-level penalties provide data-driven signals for precise structural correction, guiding the LLM toward the true underlying equation rather than mere numerical overfitting.

\begin{figure}[t]
  \centering
  \includegraphics[width=0.98\linewidth]{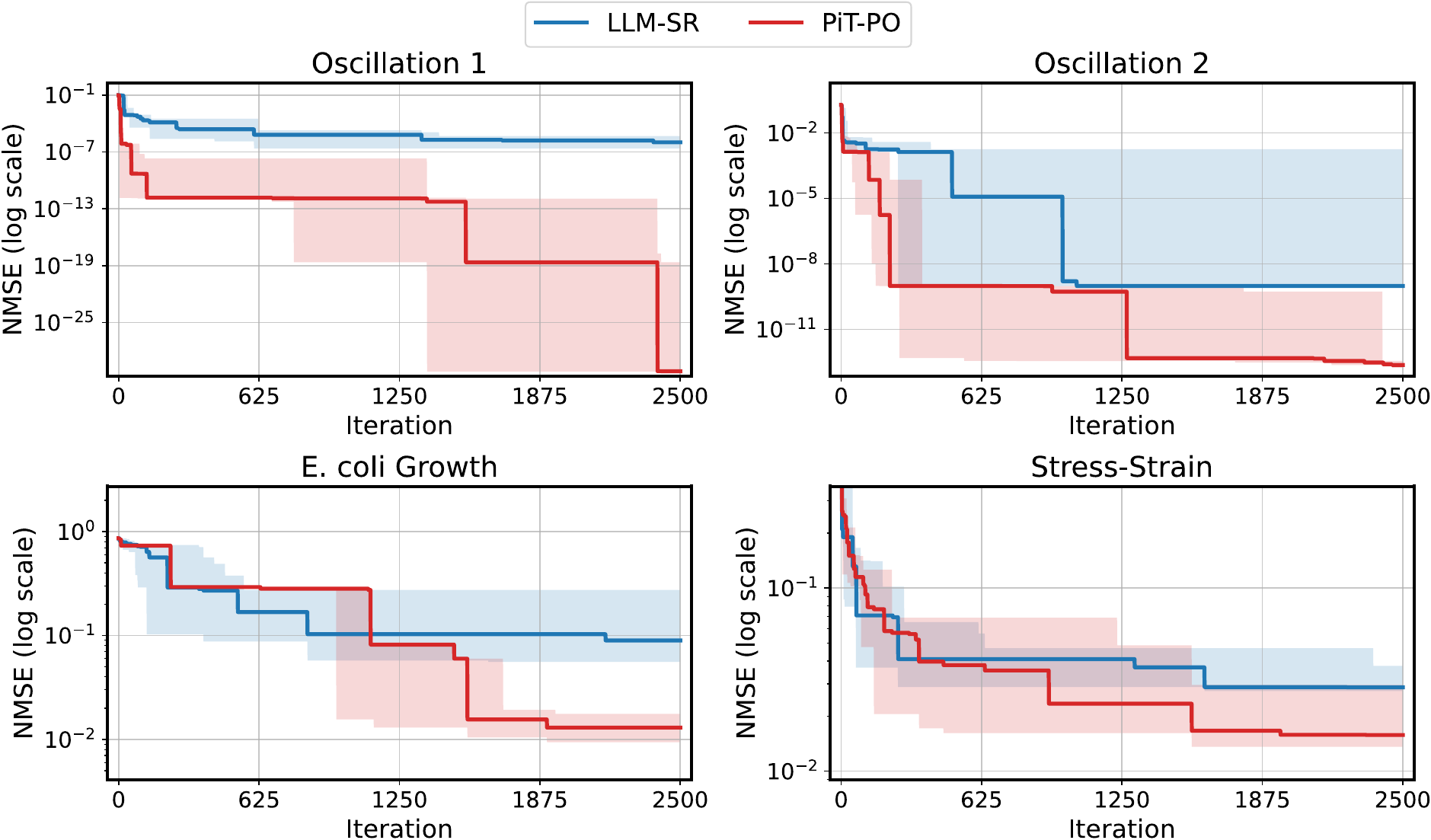}
  \caption{NMSE trajectories (log scale) over search iterations for LLM-SR and PiT-PO (Llama-3.1-8B) on LLM-SR Suite. Lines denote the median over seeds, and shaded regions indicate the min--max range.The remaining iteration curves for smaller backbones (3B and 1B) are deferred to Appendix~\ref{app:extra_curves}.
}
  \label{fig:nmse_vs_iteration}
  
\end{figure}

\begin{table*}[t]

\centering
\normalsize
\setlength{\tabcolsep}{4.0pt}
\renewcommand{\arraystretch}{1.25}

\resizebox{\textwidth}{!}{%
\begin{tabular}{l
>{\columncolor{gray!15}}c cc
>{\columncolor{gray!15}}c cc
>{\columncolor{gray!15}}c cc
>{\columncolor{gray!15}}c cc
>{\columncolor{gray!15}}c cc}
\toprule
\textbf{Models}
& \multicolumn{3}{c}{\textbf{LSR-Transform}}
& \multicolumn{12}{c}{\textbf{LSR-Synth}} \\
\cmidrule(lr){5-16}
& \multicolumn{3}{c}{}
& \multicolumn{3}{c}{Chemistry}
& \multicolumn{3}{c}{Biology}
& \multicolumn{3}{c}{Physics}
& \multicolumn{3}{c}{Material Science} \\
\cmidrule(lr){5-7}\cmidrule(lr){8-10}\cmidrule(lr){11-13}\cmidrule(lr){14-16}
& SA(\%)$\uparrow$ & Acc$_{all-0.1}$(\%)$\uparrow$ & NMSE$\downarrow$
& SA(\%)$\uparrow$ & Acc$_{all-0.1}$(\%)$\uparrow$ & NMSE$\downarrow$
& SA(\%)$\uparrow$ & Acc$_{all-0.1}$(\%)$\uparrow$ & NMSE$\downarrow$
& SA(\%)$\uparrow$ & Acc$_{all-0.1}$(\%)$\uparrow$ & NMSE$\downarrow$
& SA(\%)$\uparrow$ & Acc$_{all-0.1}$(\%)$\uparrow$ & NMSE$\downarrow$ \\
\midrule

Direct Prompting
& 3.61 & 1.801 & 0.3697
& 0.00 & 0.00 & 0.0644
& 0.00 & 0.00 & 0.5481
& 0.00 & 0.00 & 0.0459
& 0.00 & 0.00 & 0.0826 \\
SGA
& 2.70 & 0.909 & 0.3519
& 0.00 & 8.33 & 0.0458
& 0.00 & 0.00 & 0.2416
& 0.00 & 2.27 & 0.1549
& 0.00 & 12.12 & 0.0435 \\
LaSR
& 5.41 & 45.94 & \textbf{0.0021}
& 0.00 & 27.77 & 2.77e{-4}
& 4.16 & 16.66 & 2.73e{-4}
& 4.54 & 25.02 & 0.0018
& 8.21 & 64.22 & 7.44e{-5} \\
LLM-SR
& 30.63 & 38.55 & 0.0101
& 8.33 & 66.66 & 8.01e{-6}
& 25.30 & 58.33 & 1.04e{-6}
& 6.97 & 34.09 & 1.23e{-4}
& 4.10 & 88.12 & 1.15e{-7} \\
PiT-PO
& \textbf{34.23} & \textbf{46.84} & 0.0056
& \textbf{13.89} & \textbf{77.78} & \textbf{4.13e{-7}}
& \textbf{29.17} & \textbf{70.83} & \textbf{9.37e{-8}}
& \textbf{11.36} & \textbf{40.91} & \textbf{6.57e{-5}}
& \textbf{12.00} & \textbf{92.00} & \textbf{1.18e{-8}} \\
\bottomrule
\end{tabular}%
}
\caption{Overall performance on LLM-SRBench (Llama-3.1-8B-Instruct).}
\label{tab:tb2}

\end{table*}

\subsection{PiT-PO Empowers Lightweight Backbones to Rival Large Models}
As shown in Table~\ref{tab:main}, the performance of PiT-PO with the Llama-3.1-8B, Llama-3.2-3B, and Llama-3.2-1B backbones is competitive with, and often exceeds, the performance of LLM-SR that relies on substantially larger or proprietary models, including Mixtral 8$\times$7B and 4o-mini.

These results indicate that PiT-PO effectively bridges the capability gap between lightweight open-source models and large-scale commercial systems. From a practical standpoint, this reduces the barrier to entry for scientific discovery: by delivering state-of-the-art performance on consumer-grade hardware (even maintaining competitiveness with a 1B backbone), PiT-PO eliminates the dependence on massive compute and closed-source APIs, thereby democratizing access to powerful SR tools.

\begin{figure}[t]
  \centering
  \includegraphics[width=0.8\linewidth]{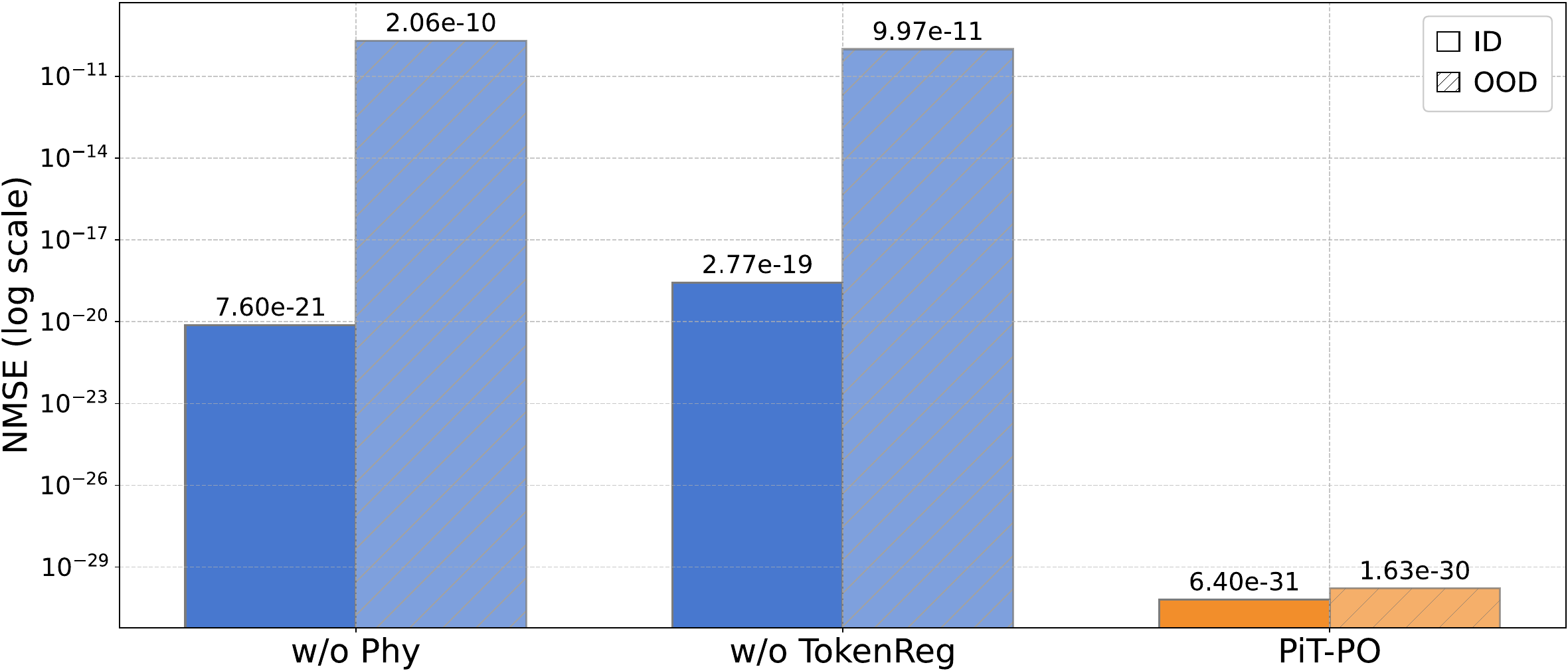}
  \caption{Ablation results of PiT-PO and its variants.}
  \label{fig:ablation}
  
\end{figure}

\subsection{PiT-PO Enhances Search Efficiency and Breaks Stagnation}
\label{4.4}
Figure~\ref{fig:nmse_vs_iteration} shows that PiT-PO achieves superior search efficiency in discovering accurate equations.
In the early search stage, the red and blue curves are close across all four tasks: both methods primarily rely on the fitting signal (MSE) and therefore exhibit comparable per-iteration progress. 
As NMSE enters a lower regime, the trajectories consistently separate: PiT-PO exhibits abrupt step-wise drops while LLM-SR tends to plateau, yielding a clear red--blue gap in every subplot. Concretely, once the search reaches these lower-error regions, PiT-PO repeatedly exits stagnation and transitions to the next accuracy phase with orders-of-magnitude NMSE reductions (most prominently in Oscillation 1 and Oscillation 2, and also evident in E. coli Growth and Stress-Strain), whereas LLM-SR often remains trapped near its current error floor. This behavior confirms that the proposed dual-constraint mechanism effectively activates exactly when naive MSE feedback becomes insufficient. By penalizing physical inconsistencies and structural redundancy, PiT-PO forces the LLM to exit stagnation and transition toward the correct functional form.

While the in-search fine-tuning introduces a computational overhead, this cost is decisively outweighed by the substantial gains in performance. As detailed in Appendix~\ref{app:time_cost}, PiT-PO maintains a significant performance edge even when evaluated under equivalent wall-clock time, demonstrating that the accelerated convergence speed effectively compensates for the additional training time.

\subsection{Ablation Study}
To rigorously validate the contribution of each algorithmic component, we conduct an ablation study across three settings: \textbf{w/o Phy}, which excludes the physics-consistency penalty $P_{\text{phy}}$; \textbf{w/o TokenReg}, which removes the redundancy-aware token-level regularization; and the full \textbf{PiT-PO} framework. As shown in Figure~\ref{fig:ablation}, removing any single component leads to a substantial deterioration of NMSE and a larger generalization gap between In-Distribution (ID) and Out-Of-Distribution (OOD) data. These empirical results underscore the necessity of the complete framework, demonstrating that the proposed dual constraints are indispensable for ensuring both search stability and robust generalization.

\begin{figure}[t]
  \centering
  \includegraphics[width=0.9\linewidth]{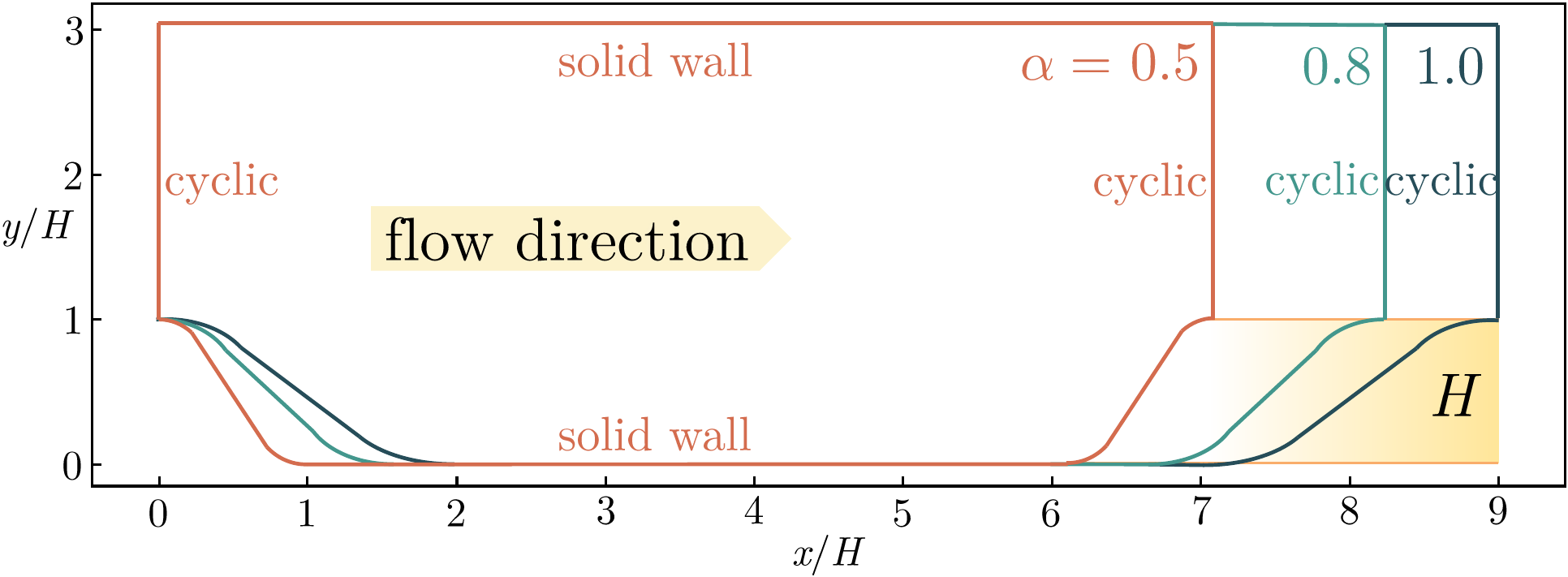}
  \caption{Schematic of the geometries for periodic hills.}
  \label{fig:set}
  
\end{figure}

\subsection{Case Study: Turbulence Modeling}

To validate the practical utility of PiT-PO in high-fidelity scientific discovery, we select the \textbf{Flow over Periodic Hills}~\cite{XIAO2020104431} (Figure~\ref{fig:set}) as a testbed. This problem is widely recognized in Computational Fluid Dynamics (CFD)~\cite{Pope2000} as a benchmark for \textit{Separated Turbulent Flows}, presenting complex features such as strong adverse pressure gradients, massive flow detachment, and reattachment.

\textit{Problem Definition and Physics:} The geometry consists of a sequence of polynomially shaped hills arranged periodically in the streamwise direction. The flow is driven by a constant body force at a bulk Reynolds number of $Re_b = 5600$ (based on hill height $H$ and bulk velocity $U_b$). The domain height is fixed at $L_y/H = 3.036$, while the streamwise length $L_x$ varies with the slope factor $\alpha$ according to $L_x/H = 3.858\alpha + 5.142$. Periodic boundary conditions are applied in the streamwise direction, with no-slip conditions on the walls.

\textit{The Scientific Challenge:} The challenge lies in the \textbf{Separation Bubble}~\cite{Pope2000}, a region where turbulence exhibits strong \textbf{anisotropy} due to streamline curvature. Traditional Linear Eddy Viscosity Models (LEVM)~\cite{Pope2000}, such as the $k$-$\omega$ SST model~\cite{menter1994two,menter2003ten}, rely on the Boussinesq hypothesis which assumes isotropic turbulence. Consequently, they systematically fail to predict key flow features, such as the separation bubble size and reattachment location.

\textit{Discovery Objective:} Instead of fitting a simple curve, our goal is to discover a \textbf{Non-linear Constitutive Relation} for the Reynolds stress anisotropy tensor $a_{ij}$ and the dimensionless Reynolds stress anisotropy tensor $b_{ij}$. By learning the Reynolds stress tensor $\tau_{ij}$ from high-fidelity Direct Numerical Simulation (DNS)~\cite{Pope2000} data, PiT-PO aims to formulate a symbolic correction term that captures the anisotropic physics missed by linear models.

\textit{Baselines:} We follow standard turbulence modeling protocols and compare primarily against the standard $k$-$\omega$ SST model of RANS. We also include LLM-SR and DSRRANS~\cite{doi:10.10635.0135638}, a strong SR-based turbulence modeling method specifically designed for turbulence tasks.

\begin{figure}[t]
  \centering
  \includegraphics[width=0.98\linewidth]{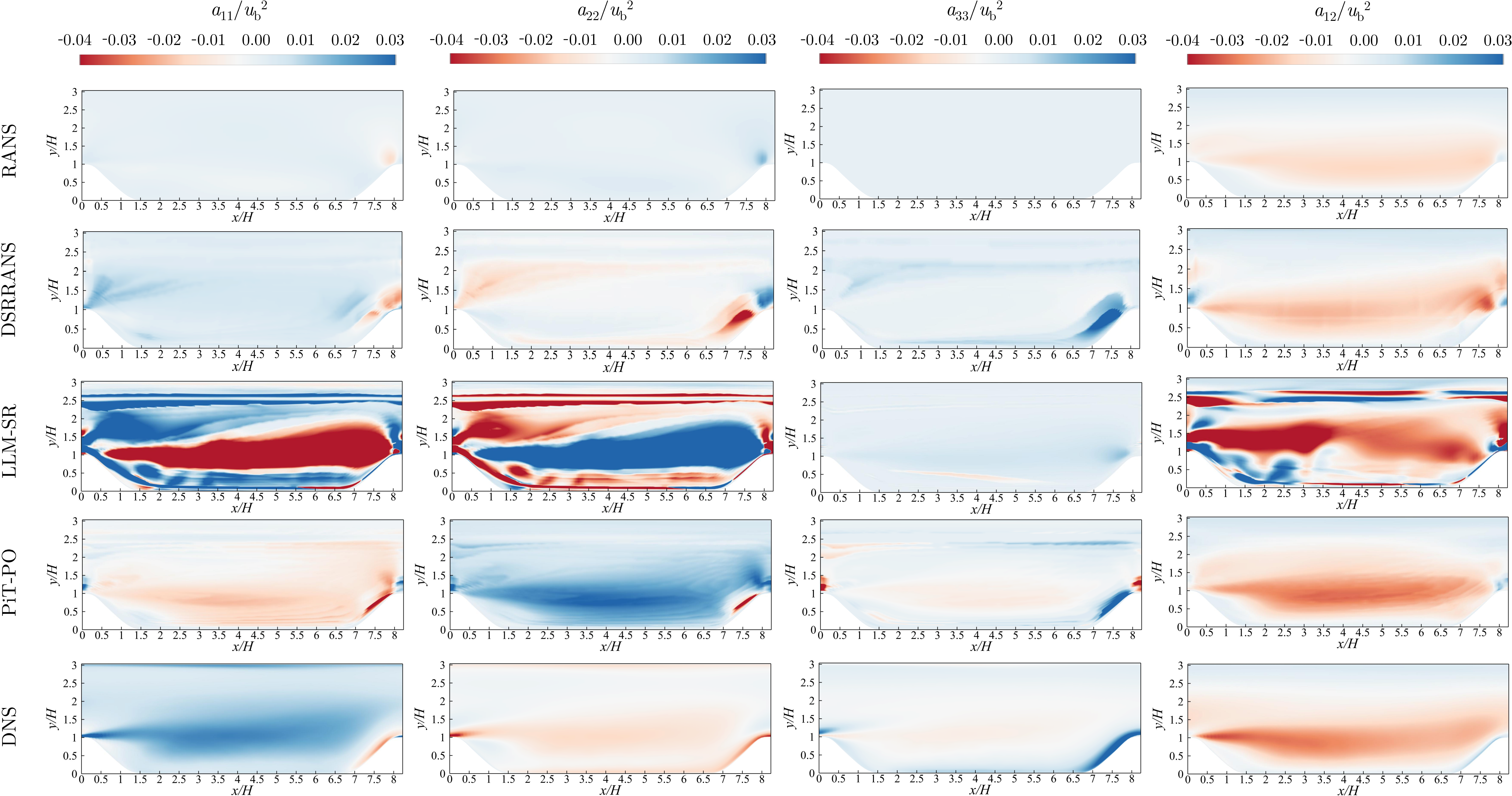}
  \caption{Comparison of the four anisotropic Reynolds stress components for periodic hill training flow 
  using RANS, DSRRANS, LLM-SR, PiT-PO and DNS, respectively. }
  
  \label{fig:aniso}
  \vspace{-1\baselineskip}
\end{figure}

We cast turbulence closure modeling as a SR problem~\cite{Tang2023} (see Appendix~\ref{sec:app_turbulence_setup} for details). 
After obtaining the final symbolic equation, we embed it into a RANS solver of OpenFOAM~\cite{OpenFOAM} and run CFD simulations on the periodic-hill configuration. 
We compare the resulting Reynolds-stress components, mean-velocity fields, and skin-friction profiles against DNS references. 
Figures~\ref{fig:aniso}--\ref{fig:cf} visualize these quantities, enabling a direct assessment of physical fidelity and flow-field prediction quality.

Based on the comparative analysis of the anisotropic Reynolds stress contours (Figure~\ref{fig:aniso}), DSRRANS and PiT-PO show enhancement over the traditional RANS approach.
Among them, PiT-PO performs the best: its contour matches the DNS reference most closely, with reduced error compared to DSRRANS and LLM-SR, demonstrating less severe non-physical extremes.

The stream-wise velocity contours illustrate the correction of the bubble size, a region of reversed flow that forms when fluid detaches from a surface. In Figure~\ref{fig:ux}, PiT-PO most accurately represents the extent and shape of the recirculation zone, where fluid circulates within the separated region, closely consistent with the DNS data throughout the domain, particularly within the separation region and the recovery layer, where flow re-attaches to the surface.

The skin friction coefficient (Figure~\ref{fig:cf}), defined as the ratio of the wall stress to the dynamic pressure of the flow along the bottom wall, is a sensitive metric for predicting flow separation. The $k$-$\omega$ SST model of RANS underestimates the magnitude of the skin friction and predicts a delayed reattachment location compared to the DNS. The learned model (PiT-PO) improves the prediction, aligning more closely with the DNS profile.

These results demonstrate that PiT-PO can generate symbolic equations tailored to turbulence modeling and that, under \emph{a posteriori} CFD evaluation, the resulting predictions more closely match DNS references, which increases the practical value of LLM-based SR in real scientific and engineering workflows. With the proposed dual constraints, 

PiT-PO provides targeted search and learning signals that enable the internalization of turbulence priors during equation discovery, thereby steering the model toward physically consistent and domain-relevant structures.

\begin{figure}[t]
  \centering
  \includegraphics[width=0.97\columnwidth]{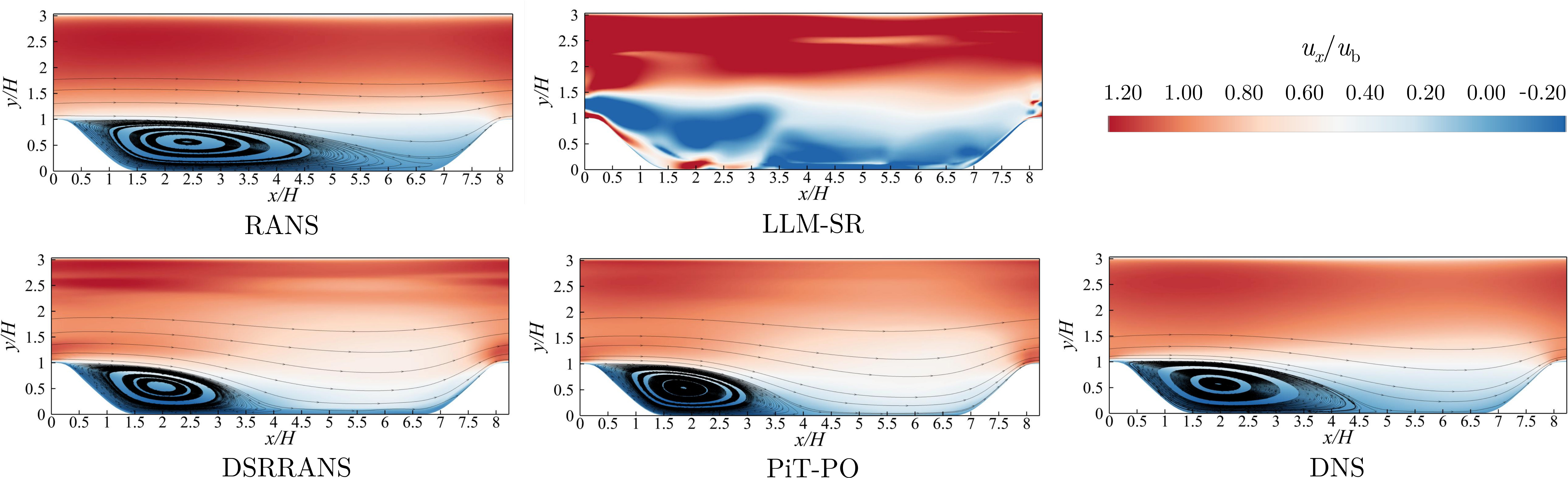}
  \caption{Non-dimensional stream-wise velocity contours obtained by the learned model and the standard $k$-$\omega$ SST model of RANS, compared with DNS data.}
  
  \label{fig:ux}
\vspace{-1\baselineskip}
\end{figure}

\section{Related Work }

Traditional SR has been studied through several lines, including genetic programming , reinforcement learning~\cite{petersen2021deepsymbolicregressionrecovering}, and transformer-based generation~\cite{Biggioetal21}. Genetic programming~\cite{130444} casts equation discovery as an evolutionary search over tree-structured programs, where candidate expressions are iteratively refined via mutation and crossover. Reinforcement learning-based SR, introduced by Petersen et al.~\cite{petersen2021deepsymbolicregressionrecovering}, has developed into a family of policy-optimization frameworks~\cite{mundhenk2021symbolicregressionneuralguidedgenetic,landajuela2021improvingexplorationpolicygradient,crochepierre:hal-03695471,du2023discoverdeepidentificationsymbolically} that formulate SR as a sequential decision-making process. More recently, transformer-based models~\cite{valipour2021symbolicgptgenerativetransformermodel,10462113,kamienny2022endtoendsymbolicregressiontransformers,Li2023TransformerbasedMF,zhang2025ragsr} have been adopted for SR, using large-scale pretraining to map numerical data directly to equations. However, these methods typically fail to incorporate scientific prior knowledge.

Recent progress in natural language processing has further enabled LLM-based SR methods, including LLM-SR~\cite{shojaee2025llmsrscientificequationdiscovery}, LaSR~\cite{grayeli2024symbolicregressionlearnedconcept}, ICSR~\cite{Merler_2024}, CoEvo~\cite{guo2025coevocontinualevolutionsymbolic}, and SR-Scientist~\cite{xia2025srscientistscientificequationdiscovery}. LLM-SR exploits scientific priors that are implicitly captured by LLMs to propose plausible functional forms, followed by data-driven parameter estimation. LaSR augments SR with abstract concept generation to guide hypothesis formation, while ICSR reformulates training examples as in-context prompts to elicit function generation. However, a unifying limitation across these methods is their reliance on the LLM as a frozen generator, which precludes incorporating search feedback to update the generation strategy and consequently restricts their ability to adapt to complex problems.

While some recent works, such as SOAR~\cite{pourcel2025selfimprovinglanguagemodelsevolutionary} and CALM~\cite{huang2025calmcoevolutionalgorithmslanguage}, have begun to explore adaptive in-search tuning, they primarily focus on algorithm discovery or combinatorial optimization problems, whereas our method is specifically tailored for SR. By integrating hierarchical physical constraints and theorem-guided token regularization, PiT-PO establishes an adaptive framework capable of discovering accurate and physically consistent equations.

\begin{figure}[t]
  \centering
  \includegraphics[width=0.95\linewidth]{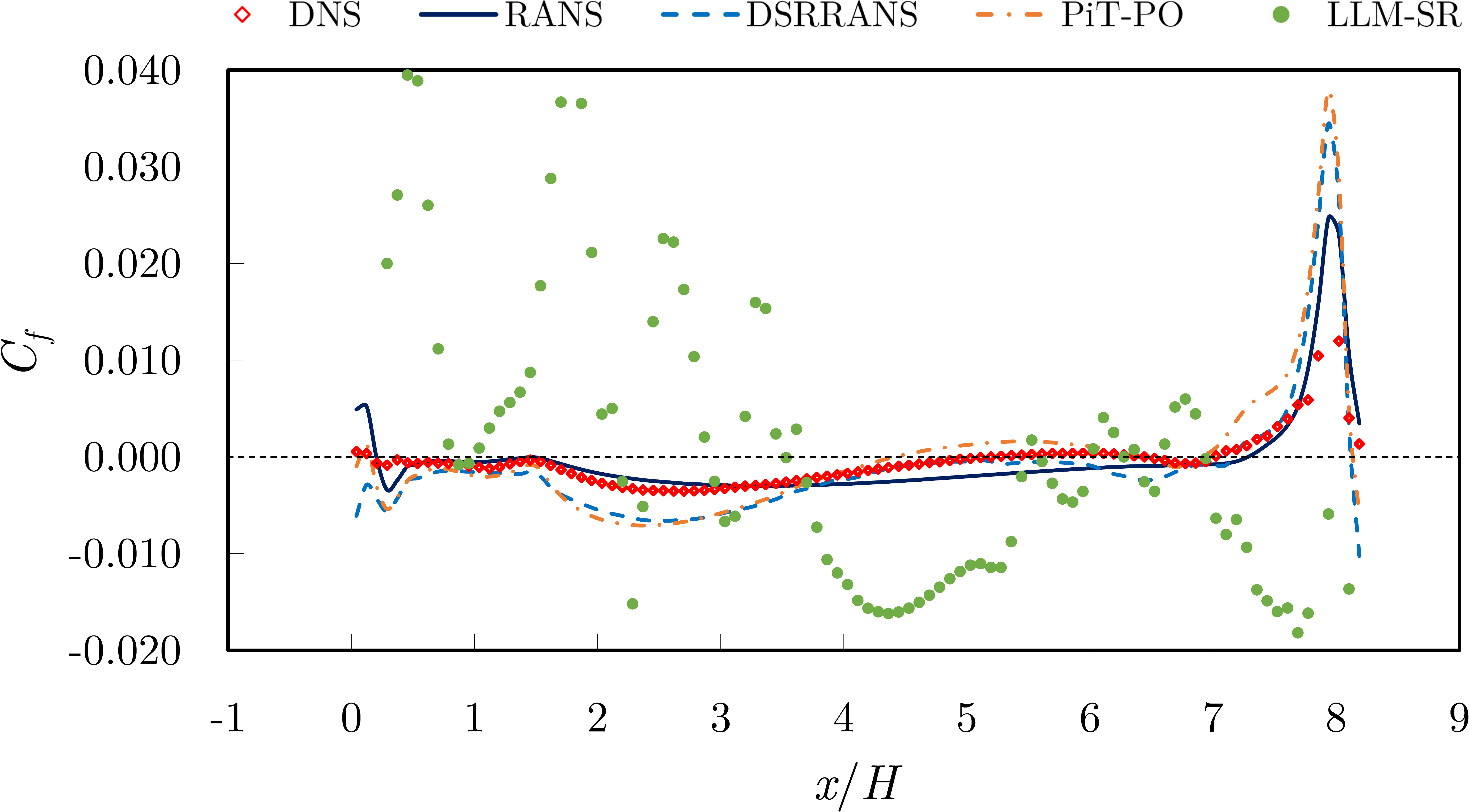}
  \caption{Skin friction distribution along the bottom obtained by the learned model and the standard $k$-$\omega$ SST model of RANS, compared with DNS data.}
  \label{fig:cf}
  % \vspace{-1\baselineskip}
\end{figure}

\section{Conclusion}

In this work, we introduced PiT-PO, a unified framework that fundamentally transforms LLMs from static equation proposers into adaptive, physics-aware generators for SR. By integrating in-search policy optimization with a novel dual-constraint evaluation mechanism, PiT-PO rigorously enforces hierarchical physical validity while leveraging theorem-guided, token-level penalties to eliminate structural redundancy. This synergistic design aligns generation with numerical fitness, scientific consistency, and parsimony, establishing new state-of-the-art performance on SR benchmarks. Beyond synthetic tasks, PiT-PO demonstrates significant practical utility in turbulence modeling, where the discovered symbolic corrections improve Reynolds stress and flow-field predictions. Notably, PiT-PO achieves these results using small open-source backbones, making it a practical and accessible tool for scientific communities with limited computational resources. Looking forward, we plan to extend PiT-PO to broader scientific and engineering domains by enriching the library of domain-specific constraints and validating it across more complex, real-world systems. 
Moreover, we anticipate that integrating PiT-PO with larger-scale multi-modal foundation models could further unlock its potential in processing heterogeneous scientific data.

\section*{Acknowledgments}
This work is supported in part by the Strategic Priority Research Program of the Chinese Academy of Sciences (Grant No. XDA0480201), the Natural Science Foundation of China (Grant Nos. 62222606, 61902402 and 62273347), the Key Research and Development Program of Jiangsu Province (Grant No. BE2023016), and the CCF-Baidu Open Fund.\footnote{The source code is publicly available at \url{https://github.com/CAS-CLab/PiT-PO}. The PaddlePaddle implementation used in this work is available at \url{https://github.com/CAS-CLab/PiT-PO/tree/main/PiT-PO_for_Paddle}.}

\newpage


\newpage
\newpage
\appendix
\begin{table*}[t]
\centering
\small
\setlength{\tabcolsep}{6pt}
\begin{tabular}{l l l p{0.4\textwidth}}
\hline
Component & Symbol & Value & Notes \\
\hline
Fitting log-MSE scale & $\alpha$ & \texttt{1.0} &
Scaling in $R_{\mathrm{fit}}=-\alpha\log(\mathrm{MSE}+\epsilon)$. \\
Complexity penalty weight & $\lambda_{\mathrm{len}}$ & \texttt{5e-3} &
Weight in $P_{\mathrm{cplx}}=\lambda_{\mathrm{len}}\cdot \mathrm{Length}(\mathrm{AST})$. \\
Dimensional homogeneity penalty weight & $P_{\mathrm{dim}}$ & \texttt{1.0} &
Penalty/reward weight for unit inconsistency (general-level constraint). \\
Differentiability penalty weight & $P_{\mathrm{diff}}$ & \texttt{0.5} &
Penalty/reward weight for non-smoothness on the data-defined domain (general-level constraint). \\
Domain-specific constraint weight (j-th) & $P_{\mathrm{domain}}^{(j)}$ & \texttt{0.5} &
Weight for the $j$-th domain prior (e.g., realizability, wall BC, asymptotic scaling, energy consistency). \\
Gating threshold for physics constraints & $\delta_{\mathrm{gate}}$ & \texttt{1e-3 * MSE\_\{initial\}} &
Activate physics penalties only after fitting reaches a sufficiently low-error regime. \\
Numerical stability constant in $\tau_i$ & $\epsilon$ & \texttt{1e-50} &
Used in $\tau_i = \frac{|b_i|}{\sum_j |b_j| + \epsilon}$ to avoid division by zero. \\
Redundancy threshold & $\rho$ & \texttt{1e-2} &
A term is considered redundant if $\tau_i \le \rho$. \\
Token penalty scale & $p$ & \texttt{0.5} &
Scaling coefficient in $P_{\mathrm{tok}} = p\cdot \max(0, -\log(|b_i|+\epsilon))$. \\
\hline
\end{tabular}
\caption{Hyperparameters for dual-constraint evaluation and token-level signal synthesis.}
\label{tab:dual_constraint_hparams}
\end{table*}

\section{Additional Hyperparameters of PiT-PO}
\label{app:pitpo_hparams}

Table~\ref{tab:dual_constraint_hparams} lists the additional hyperparameters of PiT-PO, including (i) general-level physical penalties (dimensional homogeneity and differentiability),
(ii) domain-specific constraint penalties, (iii) the gated activation threshold for physical constraints,
and (iv) theorem-guided redundancy detection and token-level penalization.

\section{Support Exclusion Theorem: Theory and Practice}
\label{app:therom}
\subsection{Preliminaries: Empirical Function Space and Orthogonality}

\begin{definition}[Basis functions / dictionary]
Let $\mathcal{X}\subseteq\mathbb{R}^d$ be the domain, and let
$\{\phi_j:\mathcal{X}\to\mathbb{R}\}_{j\in\mathcal{S}}$
be a collection of candidate basis (dictionary) functions indexed by $\mathcal{S}$.
For any subset $\mathcal{K}\subseteq\mathcal{S}$, denote the corresponding model subspace by
\[
\mathrm{span}\{\phi_j:j\in\mathcal{K}\}.
\]
\end{definition}

\begin{definition}[Target function $f^{*}$]
Assume the ground-truth target function admits a sparse expansion over the dictionary:
\[
f^{*}(x)=\sum_{j\in\mathcal{S}'} a_j\,\phi_j(x),
\]
where $\mathcal{S}'\subseteq\mathcal{S}$ is the true support set (indices of nonzero terms) and satisfies $|\mathcal{S}'|\le M$.
Moreover, the true coefficients are bounded as
\[
A \le |a_j| \le B,\qquad \forall j\in\mathcal{S}'.
\]
\end{definition}

\begin{definition}[Empirical inner product and empirical norm]
Given a finite dataset $D=\{x_n\}_{n=1}^N\subset\mathcal{X}$, for any two functions $f,g:\mathcal{X}\to\mathbb{R}$,
define the empirical inner product by
\[
\langle f,g\rangle_D := \frac{1}{N}\sum_{n=1}^N f(x_n)\,g(x_n),
\]
and the induced empirical norm by
\[
\|f\|_D := \sqrt{\langle f,f\rangle_D}.
\]
The empirical function space
\[
\mathcal{F}_D := \{(f(x_1),\ldots,f(x_N)):\ f:\mathcal{X}\to\mathbb{R}\}
\]
is a finite-dimensional inner-product space under $\langle\cdot,\cdot\rangle_D$.
\end{definition}

\begin{definition}[Empirical Orthogonality]
Two functions $f,g\in\mathcal{F}_D$ are empirically orthogonal if $\langle f,g\rangle_D = 0$.
\end{definition}

\noindent\textbf{Convention.}
Unless stated otherwise, throughout this appendix we write $\langle\cdot,\cdot\rangle$ and $\|\cdot\|$
to mean the empirical inner product $\langle\cdot,\cdot\rangle_D$ and the empirical norm $\|\cdot\|_D$.

\subsection{Proof of the Support Exclusion Theorem}

\begin{theorem}[Support Exclusion Theorem]
Let $\mathcal{K}\subseteq\mathcal{S}$ be a candidate skeleton (active set).
Consider the empirical least-squares fit over $\mathcal{K}$:
\[
b \in \arg\min_{\{b_j\}_{j\in\mathcal{K}}}\left\| f^{*}-\sum_{j\in\mathcal{K}} b_j \phi_j \right\|^2.
\]
Define the empirical Gram matrix $G\in\mathbb{R}^{|\mathcal{S}|\times|\mathcal{S}|}$ by
\[
G_{ij}:=\langle \phi_i,\phi_j\rangle,\qquad i,j\in\mathcal{S},
\]
and assume $G_{ii}>0$.
For any $i\in\mathcal{S}$, define
\[
T_{ij}:=\frac{G_{ji}}{G_{ii}},\qquad j\in\mathcal{S}.
\]
Further define the external-candidate vector
\[
u_i := \big(|T_{i\ell}|\big)_{\ell\in\mathcal{S}\setminus\mathcal{K}}
\in \mathbb{R}^{|\mathcal{S}\setminus\mathcal{K}|},
\]
and let $s(1)\ge s(2)\ge \cdots \ge s(|\mathcal{S}\setminus\mathcal{K}|)$ be the entries of $u_i$ sorted in
non-increasing order.
If for some $i\in\mathcal{K}$,
\[
|b_i|
<
A-\left(
\underbrace{\sum_{j\in\mathcal{K},\,j\ne i}(B+|b_j|)\,|T_{ij}|}_{\textnormal{Internal Interference}}
+
\underbrace{B\sum_{k=1}^{|\mathcal{S}\setminus\mathcal{K}|} s(k)}_{\textnormal{External Interference}}
\right),
\]
then $i\notin\mathcal{S}'$.
\end{theorem}

\begin{proof}
\textbf{Step 1: Empirical orthogonality (first-order optimality of least squares).}
Let
\[
g(x):=\sum_{j\in\mathcal{K}} b_j \phi_j(x),\qquad r(x):=f^{*}(x)-g(x).
\]
The objective is $\|r\|^2=\langle r,r\rangle$.
For any $i\in\mathcal{K}$, the first-order optimality condition yields
\[
0=\frac{\partial}{\partial b_i}\|r\|^2 = -2\langle r,\phi_i\rangle
\quad\Longrightarrow\quad
\langle r,\phi_i\rangle=0.
\]
Hence, the residual $r$ is empirically orthogonal to $\mathrm{span}\{\phi_i:i\in\mathcal{K}\}$.

\textbf{Step 2: Set decomposition.}
Define three disjoint sets
\[
\mathcal{T}:=\mathcal{K}\cap\mathcal{S}',\qquad
\mathcal{F}:=\mathcal{K}\setminus\mathcal{S}',\qquad
\mathcal{R}:=\mathcal{S}'\setminus\mathcal{K}.
\]
Then $\mathcal{K}=\mathcal{T}\cup\mathcal{F}$ and $\mathcal{S}'=\mathcal{T}\cup\mathcal{R}$.
Expanding the residual gives
\[
r
=
\sum_{j\in\mathcal{T}} (a_j-b_j)\phi_j
+\sum_{j\in\mathcal{R}} a_j\phi_j
-\sum_{j\in\mathcal{F}} b_j\phi_j.
\]

\textbf{Step 3: An exact coefficient identity for $i\in\mathcal{T}$.}
Fix any $i\in\mathcal{T}\subseteq\mathcal{K}$. Using $\langle r,\phi_i\rangle=0$ and writing
$G_{ji}=\langle\phi_j,\phi_i\rangle$, we obtain
\[
(a_i-b_i)G_{ii}
+\sum_{j\in\mathcal{T},\,j\ne i}(a_j-b_j)G_{ji}
+\sum_{j\in\mathcal{R}} a_j G_{ji}
-\sum_{j\in\mathcal{F}} b_j G_{ji}
=0.
\]
Dividing by $G_{ii}>0$ and using $T_{ij}=G_{ji}/G_{ii}$ yields
\[
a_i-b_i
=
-\sum_{j\in\mathcal{T},\,j\ne i}(a_j-b_j)T_{ij}
-\sum_{j\in\mathcal{R}} a_j T_{ij}
+\sum_{j\in\mathcal{F}} b_j T_{ij}.
\]

\textbf{Step 4: Upper bound via internal and external interference.}
Taking absolute values and applying the triangle inequality gives
\[
|a_i-b_i|
\le
\sum_{j\in\mathcal{T},\,j\ne i}|a_j-b_j|\,|T_{ij}|
+\sum_{j\in\mathcal{R}} |a_j|\,|T_{ij}|
+\sum_{j\in\mathcal{F}} |b_j|\,|T_{ij}|.
\]
For $j\in\mathcal{T}\subseteq\mathcal{S}'$, we have $|a_j|\le B$ and thus
\[
|a_j-b_j|\le |a_j|+|b_j|\le B+|b_j|.
\]
Therefore,
\[
|a_i-b_i|
\le
\sum_{j\in\mathcal{T},\,j\ne i}(B+|b_j|)\,|T_{ij}|
+\sum_{j\in\mathcal{F}} |b_j|\,|T_{ij}|
+ B\sum_{j\in\mathcal{R}} |T_{ij}|.
\]
Since $|b_j|\le B+|b_j|$ for all $j\in\mathcal{F}$, we further have
\[
\sum_{j\in\mathcal{F}} |b_j|\,|T_{ij}|
\le
\sum_{j\in\mathcal{F}} (B+|b_j|)\,|T_{ij}|.
\]
Combining $\mathcal{T}\setminus\{i\}$ and $\mathcal{F}$ yields the internal-interference term:
\[
|a_i-b_i|
\le
\underbrace{\sum_{j\in\mathcal{K},\,j\ne i}(B+|b_j|)\,|T_{ij}|}_{\textnormal{Internal Interference}}
+
B\sum_{j\in\mathcal{R}} |T_{ij}|.
\]
For the external term, note that $\mathcal{R}\subseteq \mathcal{S}\setminus\mathcal{K}$, and the entries
$\{|T_{i\ell}|\}_{\ell\in\mathcal{S}\setminus\mathcal{K}}$ sorted in non-increasing order are
$s(1)\ge\cdots\ge s(|\mathcal{S}\setminus\mathcal{K}|)$. Hence, for any subset $\mathcal{R}$,
\[
\sum_{j\in\mathcal{R}} |T_{ij}|
\le
\sum_{k=1}^{|\mathcal{R}|} s(k).
\]
Assuming $|\mathcal{S}'|\le M$ and that the current support is partially correct, i.e., $\mathcal{K}\cap\mathcal{S}'\neq\emptyset$, we have
$|\mathcal{R}|=|\mathcal{S}'\setminus\mathcal{K}|\le |\mathcal{S}'|-1\le M-1$.
Define $m:=\min\!\big(M-1,\ |\mathcal{S}\setminus\mathcal{K}|\big)$. Since also $|\mathcal{R}|\le |\mathcal{S}\setminus\mathcal{K}|$, it follows that $|\mathcal{R}|\le m$, and thus
\[
\sum_{j\in\mathcal{R}} |T_{ij}|
\le
\sum_{k=1}^{|\mathcal{R}|} s(k)
\le
\sum_{k=1}^{m} s(k).
\]
Thus,
\[
|a_i-b_i|
\le
\sum_{j\in\mathcal{K},\,j\ne i}(B+|b_j|)\,|T_{ij}|
+
B\sum_{k=1}^{m} s(k).
\]

\textbf{Step 5: Lower bound when $i\in\mathcal{S}'$.}
If $i\in\mathcal{S}'$, then $|a_i|\ge A$. By the reverse triangle inequality,
\[
|a_i-b_i|
\ge
\big||a_i|-|b_i|\big|
\ge
A-|b_i|.
\]

\textbf{Step 6: Contradiction.}
Assume, for contradiction, that there exists $i\in\mathcal{K}\cap\mathcal{S}'$ satisfying
\[
|b_i|
<
A-\left(
\sum_{j\in\mathcal{K},\,j\ne i}(B+|b_j|)\,|T_{ij}|
+
B\sum_{k=1}^{m} s(k)
\right).
\]
Rearranging yields
\[
A-|b_i|
>
\sum_{j\in\mathcal{K},\,j\ne i}(B+|b_j|)\,|T_{ij}|
+
B\sum_{k=1}^{m} s(k).
\]
However, Step~5 implies $|a_i-b_i|\ge A-|b_i|$, while the bound above implies
\[
|a_i-b_i|
\le
\sum_{j\in\mathcal{K},\,j\ne i}(B+|b_j|)\,|T_{ij}|
+
B\sum_{k=1}^{m} s(k),
\]
which is a contradiction. Therefore such an $i$ cannot belong to $\mathcal{S}'$, i.e., $i\notin\mathcal{S}'$.
\end{proof}

\section{Supplementary Diagnostics and Analyses}
\subsection{Remaining Iteration Curves on Smaller Backbones}
\label{app:extra_curves}
In the main text, we report iteration curves for the 8B backbone. Here we provide the remaining curves for 3B and 1B (Figure~\ref{fig:iter_small_backbones}), which corroborate the observations in Section~\ref{4.4}.

\begin{figure}[t]
\centering
\begin{subfigure}{\columnwidth}
  \centering
  \includegraphics[width=\columnwidth]{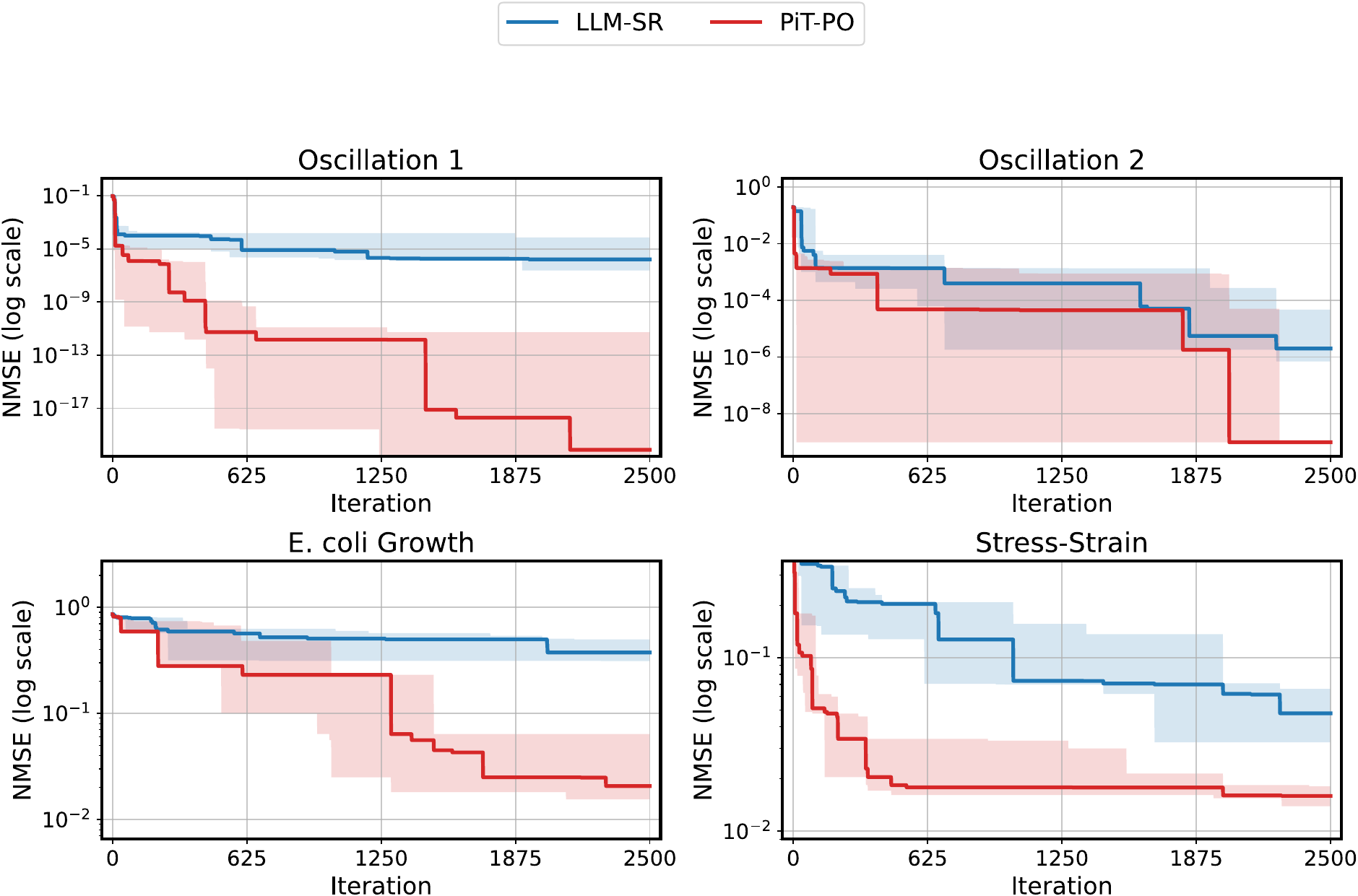} 
  \caption{Llama-3B}
  \label{fig:iter_3b}
\end{subfigure}\par
\vspace{0.35em}
\begin{subfigure}{\columnwidth}
  \centering
  \includegraphics[width=\columnwidth]{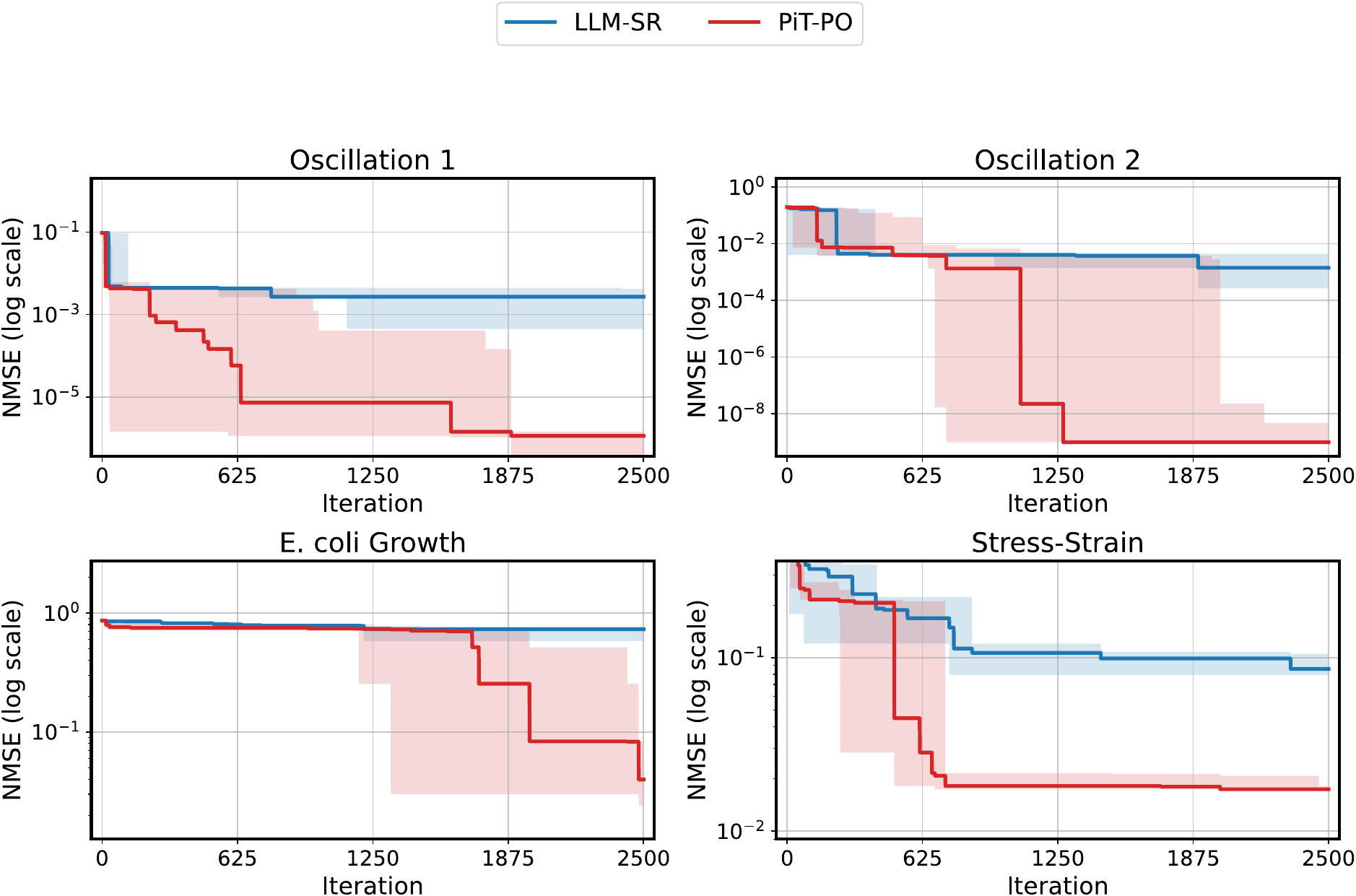} 
  \caption{Llama-1B}
  \label{fig:iter_1b}
\end{subfigure}
\caption{NMSE versus iteration on smaller backbones, consistent with the trends in the 8B setting.}
\label{fig:iter_small_backbones}
\end{figure}

\begin{figure}[t]
\centering

\begin{subfigure}{\columnwidth}
  \centering
  \includegraphics[width=\columnwidth]{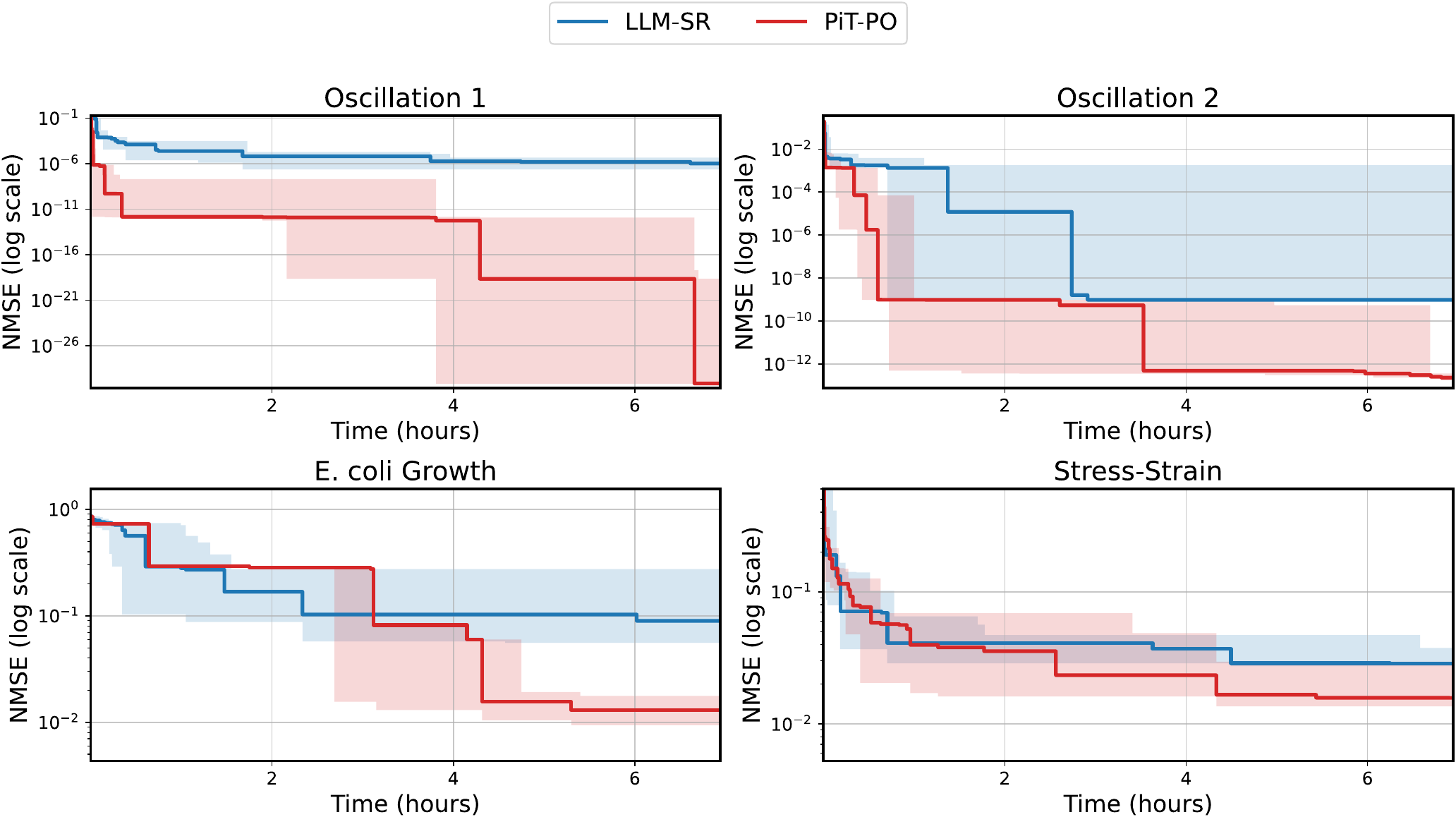}
  \caption{Llama-3.1-8B}
\end{subfigure}\par
\vspace{0.4em}

\begin{subfigure}{\columnwidth}
  \centering
  \includegraphics[width=\columnwidth]{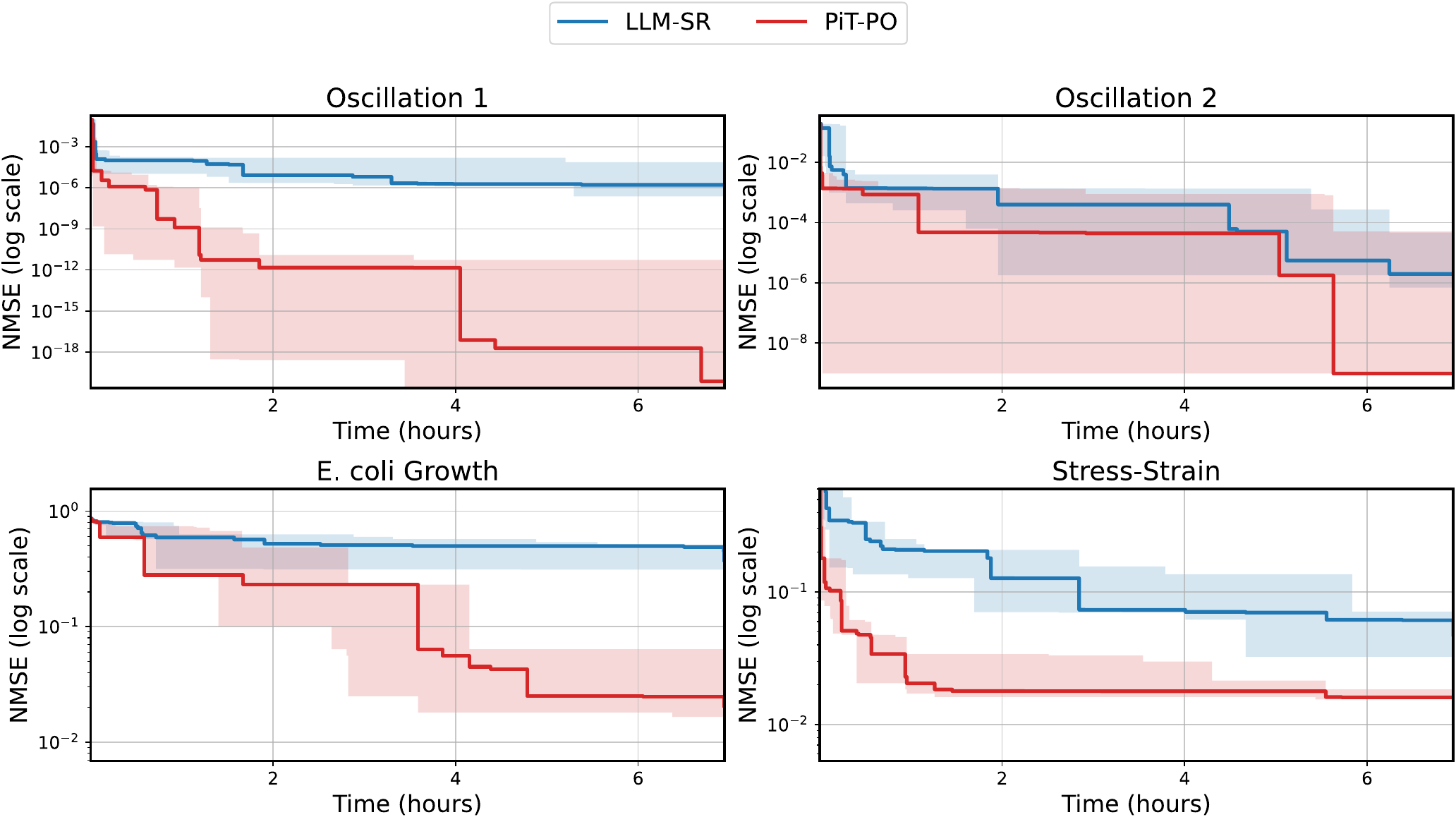}
  \caption{Llama-3.2-3B}
\end{subfigure}\par
\vspace{0.4em}

\begin{subfigure}{\columnwidth}
  \centering
  \includegraphics[width=\columnwidth]{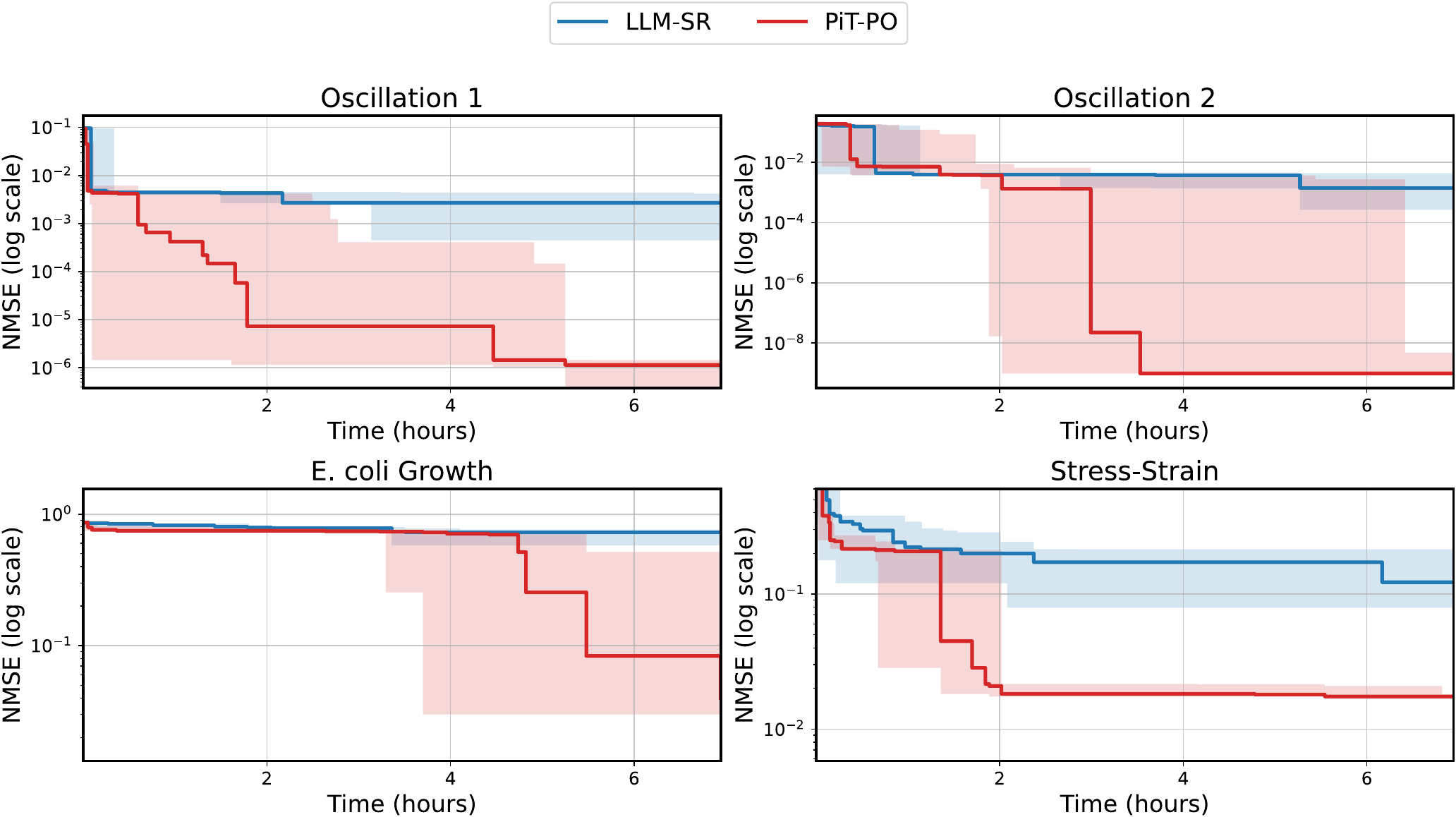}
  \caption{Llama-3.2-1B}
\end{subfigure}

\caption{Wall-clock efficiency under a fixed 25{,}000-second budget.}
\label{fig:time_cost_onecol_subfig}
\end{figure}

\subsection{Time-Cost Analysis under a Fixed Wall-Clock Budget}
\label{app:time_cost}
To complement the iteration-based efficiency analysis in Section~\ref{4.4}, we evaluate search efficiency under a strict wall-clock constraint. We fix the total runtime of each method to 25{,}000 seconds ($\approx$ 6.9 hours) and track the best-so-far NMSE as a function of elapsed time (Figure~\ref{fig:time_cost_onecol_subfig}). This protocol directly answers a practical question: given the same compute time, which method returns a more accurate recovered equation?

Across all three backbones (Llama-3.1-8B, Llama-3.2-3B, and Llama-3.2-1B), PiT-PO consistently yields lower NMSE trajectories than LLM-SR under the same 25{,}000-second budget. In line with the behavior discussed in Section~\ref{4.4}, the curves exhibit a characteristic divergence in the low-error regime: LLM-SR often plateaus after an initial reduction, while PiT-PO continues to realize step-wise decreases over time, indicating more effective late-stage credit assignment and continued progress rather than stagnation.

The advantage is most pronounced on the two oscillator systems. For 8B and 3B, PiT-PO achieves rapid early improvements and maintains additional phase transitions later in the run, reaching substantially lower NMSE within the same wall-clock budget. For the 1B model, both methods converge more slowly overall, yet PiT-PO still more reliably escapes stagnation and attains markedly lower final NMSE, suggesting that the proposed in-search tuning remains effective even in the small-model regime.

Similar trends hold for the E.\ coli Growth and Stress--Strain tasks, where progress is typically more gradual. Under equal wall-clock time, PiT-PO achieves a lower final NMSE and exhibits clearer late-stage improvements, whereas LLM-SR tends to flatten earlier. Importantly, these gains are obtained without increasing runtime: despite the additional overhead introduced by in-search policy optimization, PiT-PO delivers consistently better return-on-compute within a fixed time budget, validating the practical efficiency of the proposed approach.

\begin{table}[t]
\centering
\setlength{\tabcolsep}{5pt}
\renewcommand{\arraystretch}{1.12}
\resizebox{\columnwidth}{!}{%
\begin{tabular}{lcccc}
\hline
\multirow{2}{*}{\textbf{Method}} &
\multicolumn{1}{c}{\textbf{Oscillator1}} &
\multicolumn{1}{c}{\textbf{Oscillator2}} &
\multicolumn{1}{c}{\textbf{E. coli growth}} &
\multicolumn{1}{c}{\textbf{Stress--Strain}} \\
\cline{2-5}
& {NMSE$_{\text{OOD}}\downarrow$} & {NMSE$_{\text{OOD}}\downarrow$} & {NMSE$_{\text{OOD}}\downarrow$} & {NMSE$_{\text{OOD}}\downarrow$} \\
\hline
LLM-SR (Llama-3.1-8B) & 0.0114 & 8.05e{-10} & 0.1676    & 0.1026 \\
LLM-SR (Llama-3.2-3B) & 0.0191 & 5.56e{-5}  & 809.94  & 0.0264 \\
LLM-SR (Llama-3.2-1B) & 0.1182 & 7.57e{-3}  & 1888.5 & 0.2689 \\
\hline
PiT-PO (Llama-3.1-8B) & \cellcolor[HTML]{D9D9FF}\textbf{1.63e{-30}} & \cellcolor{gray!20}\textbf{1.36e{-11}} & 0.1163 & 0.0163 \\
PiT-PO (Llama-3.2-3B) & \cellcolor[HTML]{D9D9FF}\textbf{2.63e{-30}} & 8.52e{-10} & \cellcolor{gray!20}\textbf{0.0237} & 0.0144 \\
PiT-PO (Llama-3.2-1B) & 5.99e{-5}  & 8.58e{-10} & 0.3462 & \cellcolor{gray!20}\textbf{0.0124} \\
\hline
\end{tabular}}
\caption{OOD evaluation: NMSE on the OOD split.}
\label{tab:ood_nmse_only}
\end{table}

\subsection{OOD Evaluation (NMSE$_{\text{OOD}}$)}
We report out-of-distribution performance using NMSE on the OOD split (NMSE$_{\text{OOD}}$), summarized in Table~\ref{tab:ood_nmse_only}. Across all three backbones (8B/3B/1B) and all four tasks, PiT-PO achieves lower NMSE$_{\text{OOD}}$ than LLM-SR under the same setting. In particular, PiT-PO maintains extremely small OOD errors on the two oscillator systems, and yields substantially reduced OOD NMSE on E. coli growth and Stress--Strain, where LLM-SR exhibits noticeably larger errors. These results provide consistent evidence that the improvements observed in the main text persist under OOD evaluation.

\subsection{Equation Iterative Trajectories}
\label{app:eq_case_study}
\begin{figure}[t]
    \centering
    \begin{subfigure}{\columnwidth}
        \centering
        \includegraphics[width=\columnwidth]{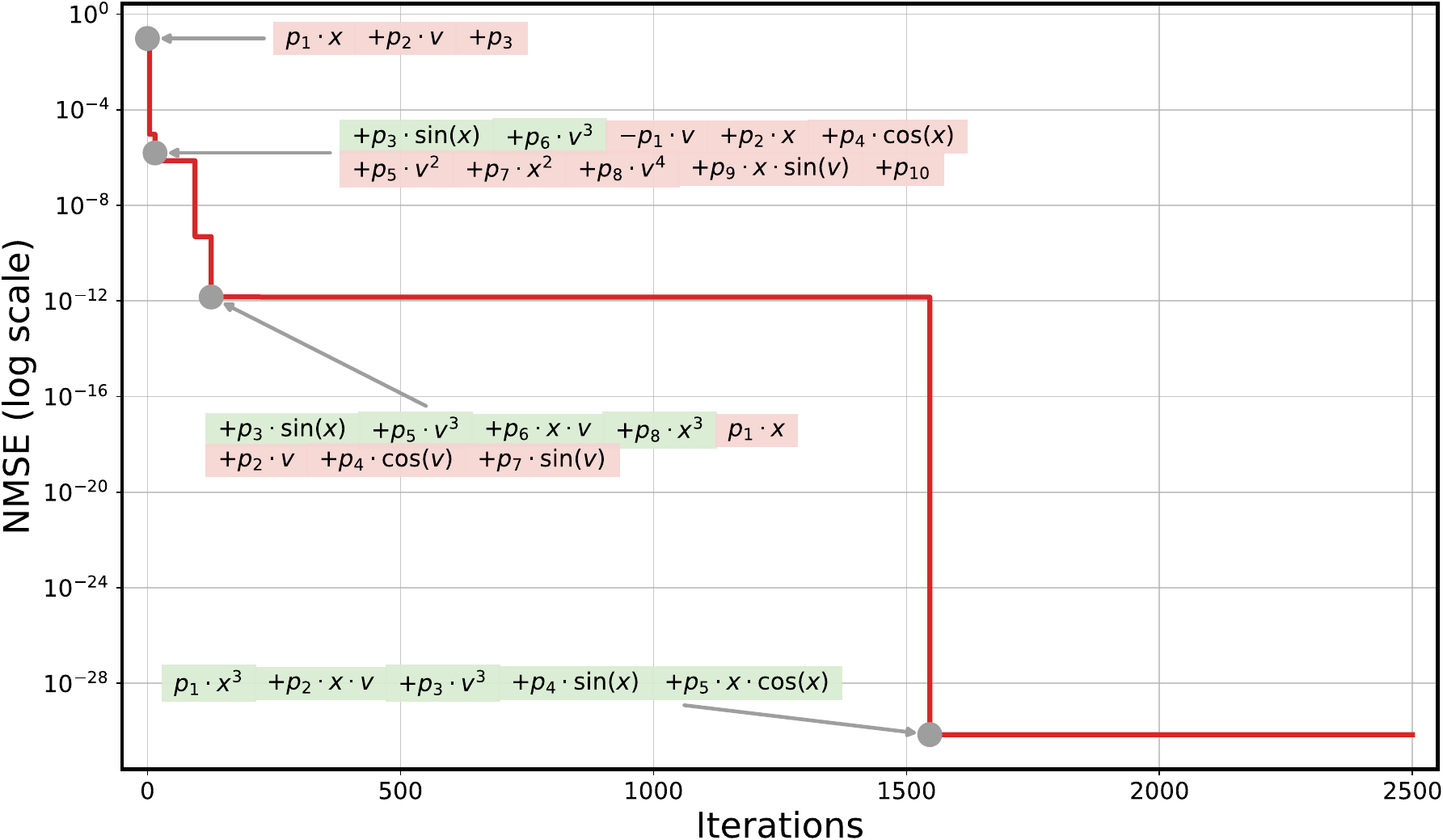}
        \caption{PiT-PO.}
        \label{fig:traj_osc1_pitpo}
    \end{subfigure}
    \vspace{0.6em}
    \begin{subfigure}{\columnwidth}
        \centering
        \includegraphics[width=\columnwidth]{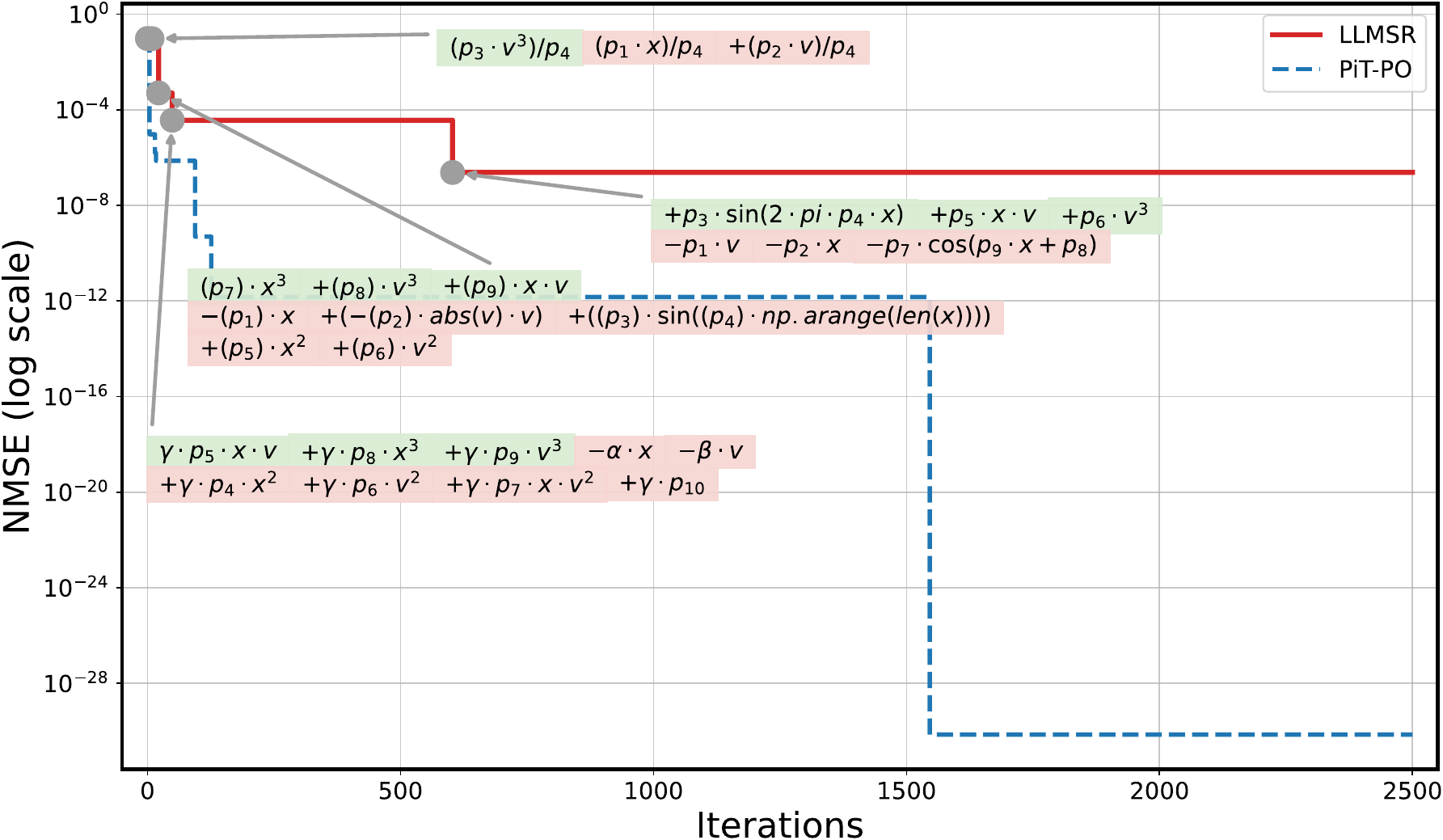}
        \caption{LLM-SR.}
        \label{fig:traj_osc1_llmsr}
    \end{subfigure}
    \caption{Iterative trajectories of PiT-PO and LLM-SR on Oscillator1 with Llama-3.1-8B-Instruct. The NMSE curves (log scale) report the best-so-far value over iterations. Annotated equation snapshots illustrate how the discovered structure evolves along the search trajectory. The green-shaded terms are correct, whereas the red-shaded terms are erroneous.}
    \label{fig:eq_iter_traj_osc1}
\end{figure}

Figure~\ref{fig:eq_iter_traj_osc1} visualizes how symbolic structures evolve along the iterative search. For PiT-PO (Figure~\ref{fig:traj_osc1_pitpo}), the trajectory reveals a clear process of progressively excluding spurious terms. Under the theorem-guided mathematical constraints and the token-aware policy update, tokens marked as redundant are assigned token-level penalties; as this penalty signal repeatedly accumulates across many sampled equations during optimization, such spurious components become progressively less preferred and are unlikely to reappear in later generations. As a result, the search does not merely fit numerically; instead, it repeatedly transitions to cleaner structures, and each step-wise NMSE drop aligns with the removal of nuisance terms and a move toward a more parsimonious functional form. This ``error-term elimination'' behavior effectively shrinks the search space and substantially strengthens the exploration capability.

In contrast, LLM-SR (Figure~\ref{fig:traj_osc1_llmsr}) exhibits much weaker structural correction. Since its guidance is primarily prompt-level rather than parameter-level, terms that appear early---even if structurally incorrect---tend to persist and continue influencing subsequent proposals. Consequently, the search frequently enters a plateau where NMSE stops improving meaningfully: the method keeps revisiting or retaining earlier spurious components instead of systematically pruning them, indicating a stagnation regime with limited ability to reach the correct structure.

\section{\textbf{Datasets}}
\label{appendix:datasets}

\subsection{LLM-SR Suite}
\label{appendix:llmsr-suite}

We use the \textsc{LLM-SR Suite}, which consists of four standard scientific equation discovery tasks spanning physics, biology, and materials science. The suite includes two nonlinear oscillator systems (Oscillation 1/2), an \textit{E. coli} growth model (E. coli Growth), and a temperature-dependent stress--strain dataset (Stress--Strain). The first three tasks are dynamical systems and are simulated over a fixed time horizon, while the last task is a static constitutive relation evaluated on experimental measurements.

\subsubsection{Oscillation 1 (Nonlinear Oscillator)}
Nonlinear oscillators with dissipative and non-polynomial couplings provide a demanding test for recovering interacting terms from trajectories. We define the state as \(x(t)\) and \(v(t)=\dot{x}(t)\), and simulate the following system:
\[
\begin{aligned}
\dot{x} &= v,\\
\dot{v} &= F \sin(\omega x) - \alpha v^3 - \beta x^3 - \gamma x v - x \cos(x).
\end{aligned}
\]
We set \(F=0.8\), \(\alpha=0.5\), \(\beta=0.2\), \(\gamma=0.5\), and \(\omega=1.0\), with initial conditions \(x(0)=0.5\), \(v(0)=0.5\), and simulate over \(t\in[0,50]\).

\subsubsection{Oscillation 2 (Nonlinear Oscillator)}
The second oscillator introduces explicit time forcing and an exponential nonlinearity, leading to a different coupling pattern between \(x\) and \(v\):
\[
\begin{aligned}
\dot{x} &= v,\\
\dot{v} &= F \sin(\omega t) - \alpha v^3 - \beta x v - \delta x \exp(\gamma x).
\end{aligned}
\]
We use \(F=0.3\), \(\alpha=0.5\), \(\beta=1.0\), \(\delta=5.0\), \(\gamma=0.5\), and \(\omega=1.0\). The initial conditions and simulation window are the same as Oscillation 1: \(x(0)=0.5\), \(v(0)=0.5\), \(t\in[0,50]\).

\subsubsection{E. coli Growth (Bacterial Growth)}
This task models the growth rate of \textit{E. coli} as a product of multiplicative factors capturing distinct physiological effects from population size, nutrient availability, temperature, and acidity:
\[
\frac{dB}{dt} = f_B(B)\cdot f_S(S)\cdot f_T(T)\cdot f_{\mathrm{pH}}(\mathrm{pH}),
\]
where \(B\) is bacterial density, \(S\) is nutrient concentration, \(T\) is temperature, and \(\mathrm{pH}\) denotes acidity. In our benchmark, the explicit functional form is:
\[
\resizebox{\columnwidth}{!}{$
\frac{dB}{dt} = \mu_{\text{max}} B \left( \frac{S}{K_S + S} \right)
\left(\frac{\tanh{k(T - x_0)}}{1 + c(T - x_{\text{decay}})^4}\right)
\exp\left(-|\text{pH} - \text{pH}_{\text{opt}}|\right)
\sin\left( \frac{(\text{pH} - \text{pH}_{\text{min}})\pi}{\text{pH}_{\text{max}} - \text{pH}_{\text{min}}} \right)^2.
$}
\]
This specification combines saturation kinetics, sharp temperature modulation, and a structured pH response, yielding a challenging multivariate nonlinear discovery setting.

\subsubsection{Stress--Strain (Material Stress Behavior)}
To assess performance on experimental measurements, we consider tensile stress--strain data for Aluminum 6061-T651 collected at multiple temperatures (from room temperature up to \(300^\circ\text{C}\)). While no single exact governing equation is provided, a commonly used phenomenological approximation for temperature-dependent hardening is:
\[
\sigma = \left( A + B \varepsilon^n \right)
\left( 1 - \left( \frac{T - T_r}{T_m - T_r} \right)^m \right).
\]
where \(\sigma\) is stress, \(\varepsilon\) is strain, \(T\) is temperature, \(T_r\) is a reference temperature, and \(T_m\) is the melting temperature. The coefficients \(A\), \(B\), \(n\), and \(m\) are fitted from data for the given alloy.

\subsection{LLM-SRBench}
\label{appendix:llmsrbench}

We additionally report results on \textsc{LLM-SRBench}, a benchmark specifically constructed to evaluate \emph{data-driven} scientific equation discovery with LLMs while reducing the risk of trivial memorization of canonical textbook formulas. It contains 239 problems spanning four scientific domains (chemistry, biology, physics, and material science), and each task is packaged with (i) a short scientific context describing variables and the target quantity and (ii) tabular numerical samples used for discovery.

\paragraph{Dataset composition}
\textsc{LLM-SRBench} is organized into two complementary subsets.

\textbf{(1) LSR-Transform.} This subset converts well-known physics equations into less-common but analytically equivalent forms by changing the prediction target. Concretely, starting from the original equation, one input variable is chosen as a \emph{pivot} and promoted to the new target; the equation is then symbolically solved (e.g., via \texttt{SymPy}) for that pivot variable, yielding an alternative representation of the same underlying physical law. Only transformations that admit an analytic solution are retained, and the original datapoints are filtered to satisfy the valid domain of the transformed expression (e.g., avoiding invalid denominators or square-root domains). Finally, a new natural-language problem statement is generated to match the transformed input--output specification. This pipeline produces 111 transformed problems. 

\textbf{(2) LSR-Synth.} This subset is designed to test discovery beyond memorized templates by composing equations from both \emph{known} scientific terms and \emph{synthetic} (novel yet plausible) terms. Candidate known/synthetic terms are proposed with an LLM given the domain context, combined into full equations, and then filtered through multiple checks: numerical solvability (e.g., via standard ODE solvers for dynamical systems), novelty assessment in context (using an LLM-based novelty evaluator), and expert plausibility validation. After passing these filters, datapoints are generated for each approved equation, yielding 128 synthetic problems across the same four domains. 

\paragraph{Symbolic Accuracy (SA)}
Besides numeric fit metrics (e.g., \(\mathrm{NMSE}\) and \(\mathrm{Acc}_{all}(\tau)\)), \textsc{LLM-SRBench} emphasizes \emph{symbolic correctness}. Symbolic Accuracy (SA) is computed using an LLM-based equivalence judge (GPT-4o): given a predicted hypothesis and the ground-truth expression, the evaluator first normalizes both by removing extraneous text (e.g., comments) and replacing explicit numeric constants with placeholder parameters, then decides whether there exist parameter values that make the hypothesis mathematically equivalent to the ground truth. SA is the fraction (percentage) of problems judged equivalent under this protocol. The benchmark authors validated this evaluation by comparing against human judgments on 130 sampled problems and reported 94.6\% agreement between GPT-4o and human evaluators.

\section{Turbulence Task Setup: Periodic Hill Flow}
\label{sec:app_turbulence_setup}

\subsection{Dataset}
\label{subsec:turb_dataset}

We consider the periodic hill flow configuration with streamwise periodic boundary conditions and no-slip walls.

In this work, the features selected for the symbolic regression process are computed from the mean strain-rate ($S$)
and rotation-rate ($R$) tensors. Following established practice, these tensors are normalized using
local turbulence scales prior to invariant calculation to ensure well-conditioned inputs:
\begin{align}
S &= \frac{1}{2}\,\frac{1}{\beta^{*}\omega}\Bigl[\nabla \boldsymbol{u} + (\nabla \boldsymbol{u})^{\mathsf{T}}\Bigr], \label{eq:turb_S_def}\\
R &= \frac{1}{2}\,\frac{1}{\beta^{*}\omega}\Bigl[\nabla \boldsymbol{u} - (\nabla \boldsymbol{u})^{\mathsf{T}}\Bigr], \label{eq:turb_R_def}
\end{align}
where $(\cdot)^{\mathsf{T}}$ denotes matrix transposition. The invariants are then computed as
$I_1=\mathrm{Tr}(S^2)$ and $I_2=\mathrm{Tr}(R^2)$.

The dataset provides (i) two invariant features $(I_1, I_2)$ (stored in the \texttt{Lambda} columns),
(ii) tensor bases $\{T^m\}_{m=1}^{3}$ (stored in the first three $3\times 3$ tensors in the \texttt{Tensor} columns),
and (iii) the target anisotropy tensor field $b$ (and its nonlinear correction form used for training, denoted by $b^\perp$),
together with auxiliary fields such as turbulent kinetic energy $k$ when needed by the evaluator.
\subsubsection{Data Collection}
\label{subsubsec:turb_data_collection}

\paragraph{DNS source.}
The DNS data are obtained from the publicly available NASA turbulence modeling resource
(\emph{Turbulence Modeling Resource / TMBWG}) for periodic hills of parameterized geometries.\footnote{
\url{https://tmbwg.github.io/turbmodels/Other_DNS_Data/parameterized_periodic_hills.html}
}

\paragraph{RANS baseline generation (for non-DNS methods).}
For all non-DNS methods, we compute baseline solutions using OpenFOAM with the same initial and boundary conditions as the DNS.
Specifically, we employ the incompressible steady-state solver \texttt{simpleFoam} and apply the turbulence models considered in this paper
(all belonging to RANS and its SST variants).
The solver outputs, at each sampling point and converged iteration/time-step, the required physical quantities for subsequent post-processing.

\paragraph{DNS--RANS grid alignment and target construction.}
To construct training targets consistent with the RANS discretization, we interpolate the DNS fields onto the RANS grid points.
The interpolated DNS quantities are then post-processed \emph{in the CFD workflow} (i.e., as offline preprocessing prior to training) to compute the Reynolds stress tensor, which is packaged into the dataset files as the supervision signal for training.

\paragraph{Feature \& basis construction.}
Following Pope's tensor-basis expansion hypothesis, the Reynolds stress decomposition involves (i) tensor bases and (ii) scalar coefficient functions
of invariant features.
Since our final turbulence closure is embedded as a correction to the baseline \texttt{kOmegaSST} model, we use the physical fields produced by
the baseline \texttt{kOmegaSST} computation for post-processing to obtain the input features and the tensor bases.
These are stored in separate files (features and bases), while the learning target is defined to match the Reynolds-stress-related quantity
described in Section~\ref{subsec:turb_to_sr}.

\subsubsection{Pre-processing}
\label{subsubsec:turb_preprocess}
We normalize the two invariant inputs prior to feeding them into the learned symbolic expressions:
\begin{equation}
\tilde{I} = \tanh\!\left(\frac{I}{2.0}\right), \qquad
\tilde{I}_1 = \tanh\!\left(\frac{I_1}{2.0}\right), \quad
\tilde{I}_2 = \tanh\!\left(\frac{I_2}{2.0}\right),
\label{eq:turb_invariant_norm}
\end{equation}
which maps the raw values into approximately $[-1, 1]$ and improves numerical stability during parameter fitting.

\subsection{Casting Turbulence Closure as a Symbolic Regression Problem}
\label{subsec:turb_to_sr}

\paragraph{Predicting scalar coefficients instead of tensor entries.}
Rather than directly searching over tensor-valued symbolic forms, we let the LLM predict three scalar coefficient
functions $(G^1, G^2, G^3)$, each parameterized as a symbolic expression of only two variables $(I_1, I_2)$.
The evaluator then reconstructs the predicted tensor via the (fixed) tensor bases:
\begin{equation}
\hat{b} \;=\; G^1(I_1,I_2)\,T^1 \;+\; G^2(I_1,I_2)\,T^2 \;+\; G^3(I_1,I_2)\,T^3.
\label{eq:turb_reconstruct_bhat}
\end{equation}
This design compresses the LLM input variable dimension to two, significantly reducing the symbolic search space.

\paragraph{Scoring by MSE in the evaluator.}
Given a candidate hypothesis, we compute the fitting score by the mean squared error between the target tensor
and the reconstructed tensor. In the final evaluator, we additionally scale both prediction and target by the sample-wise
$k$ before error aggregation:
\begin{equation}
\hat{\tau}_{ij}(x_n) = k(x_n)\,\hat{b}_{ij}(x_n), \qquad
\tau_{ij}(x_n) = k(x_n)\,b_{ij}(x_n),
\label{eq:turb_k_scaling}
\end{equation}
We compute the MSE only over a selected subset of tensor components.
This is because the Reynolds-stress (and anisotropy) tensor is symmetric, so only six components are independent,
and in the present periodic hill flow setting two off-plane shear components are typically several orders of magnitude smaller
and are commonly neglected in practical turbulence modeling.
Following this standard practice, we evaluate the error only on

\begin{equation}
\Omega = \{(0,0),(0,1),(1,1),(2,2)\},
\label{eq:turb_mse_selected}
\end{equation}

\begin{equation}
 \qquad
\mathrm{MSE} = \frac{1}{|\Omega|}\sum_{(i,j)\in\Omega}\frac{1}{N}\sum_{n=1}^{N}
\Bigl(\tau_{ij}(x_n)-\hat{\tau}_{ij}(x_n)\Bigr)^2.
\end{equation}
The resulting MSE is used as the numerical fitting score for the symbolic regression loop.

\paragraph{Prompt and fairness across methods.}
We provide the same problem specification text (task description, variable definitions, and operator constraints)
to both \textsc{LLM-SR} and \textsc{PiT-PO} to ensure a fair comparison. An example prompt is shown in
Figure~\ref{fig:spec1}--\ref{fig:spec3}. In our current experiments, we also explicitly restrict the hypothesis space by
asking the LLM to avoid trigonometric functions.

\begin{figure}[t]
  \centering
  \adjustbox{max totalsize={0.45\textwidth}{0.85\textheight},center}{%
    \includegraphics{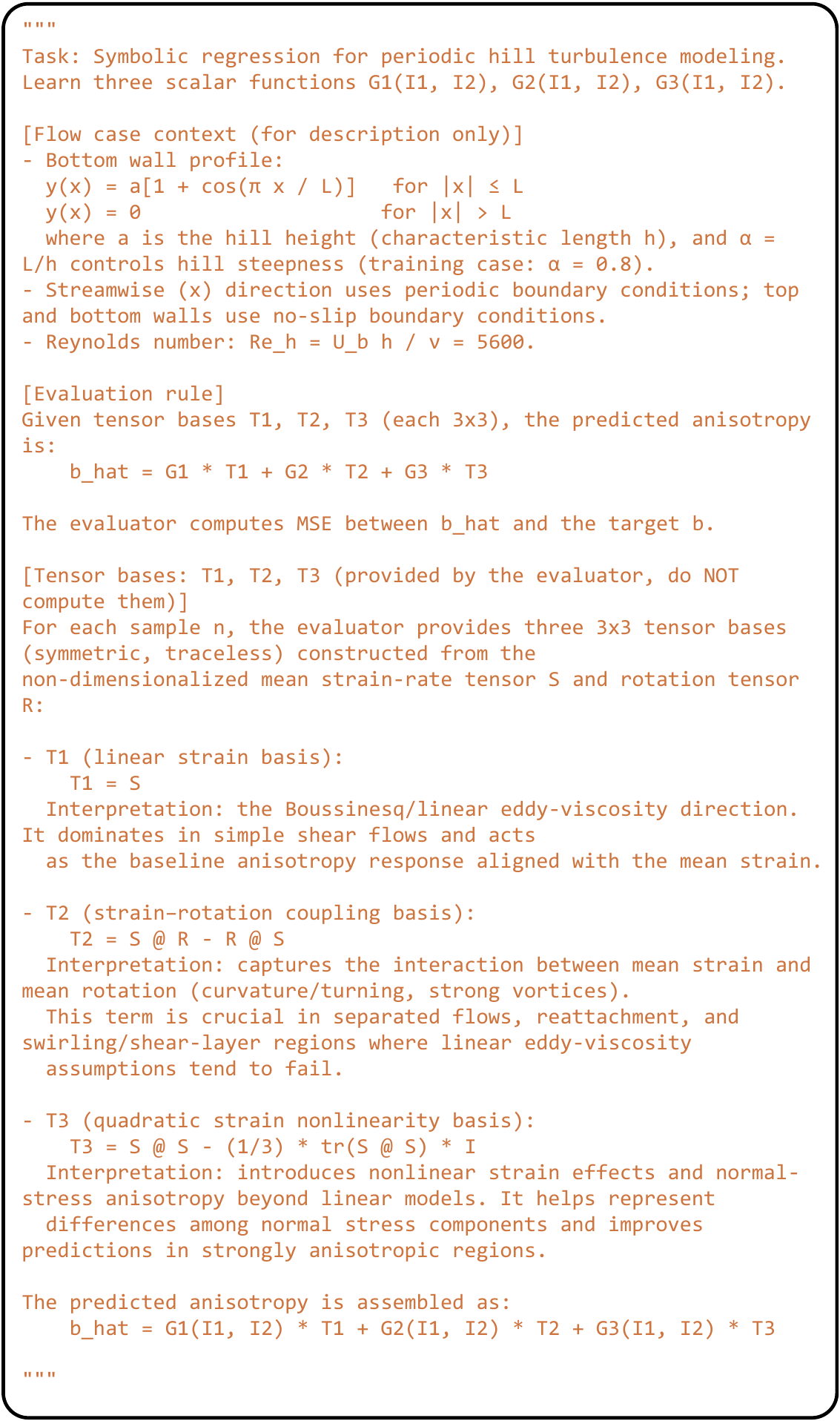}%
  }
  \caption{Problem Specification.}
  \label{fig:spec1}
\end{figure}

\begin{figure}[t]
  \centering
  \adjustbox{max totalsize={0.45\textwidth}{0.85\textheight},center}{%
    \includegraphics{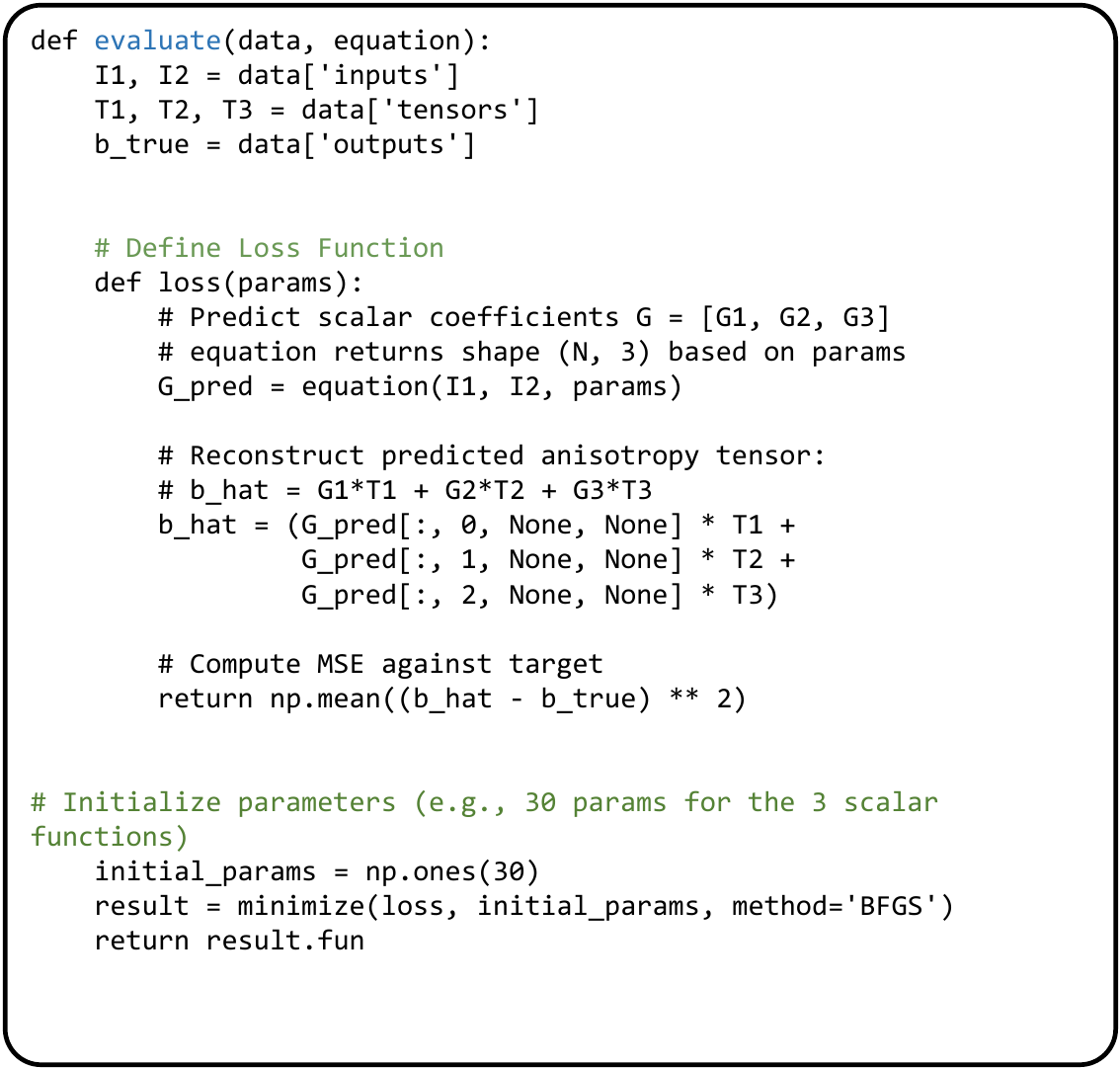}%
  }
  \caption{Evaluation and Optimization.}
  \label{fig:spec2}
\end{figure}

\begin{figure}[t]
  \centering
  \adjustbox{max totalsize={0.45\textwidth}{0.85\textheight},center}{%
    \includegraphics{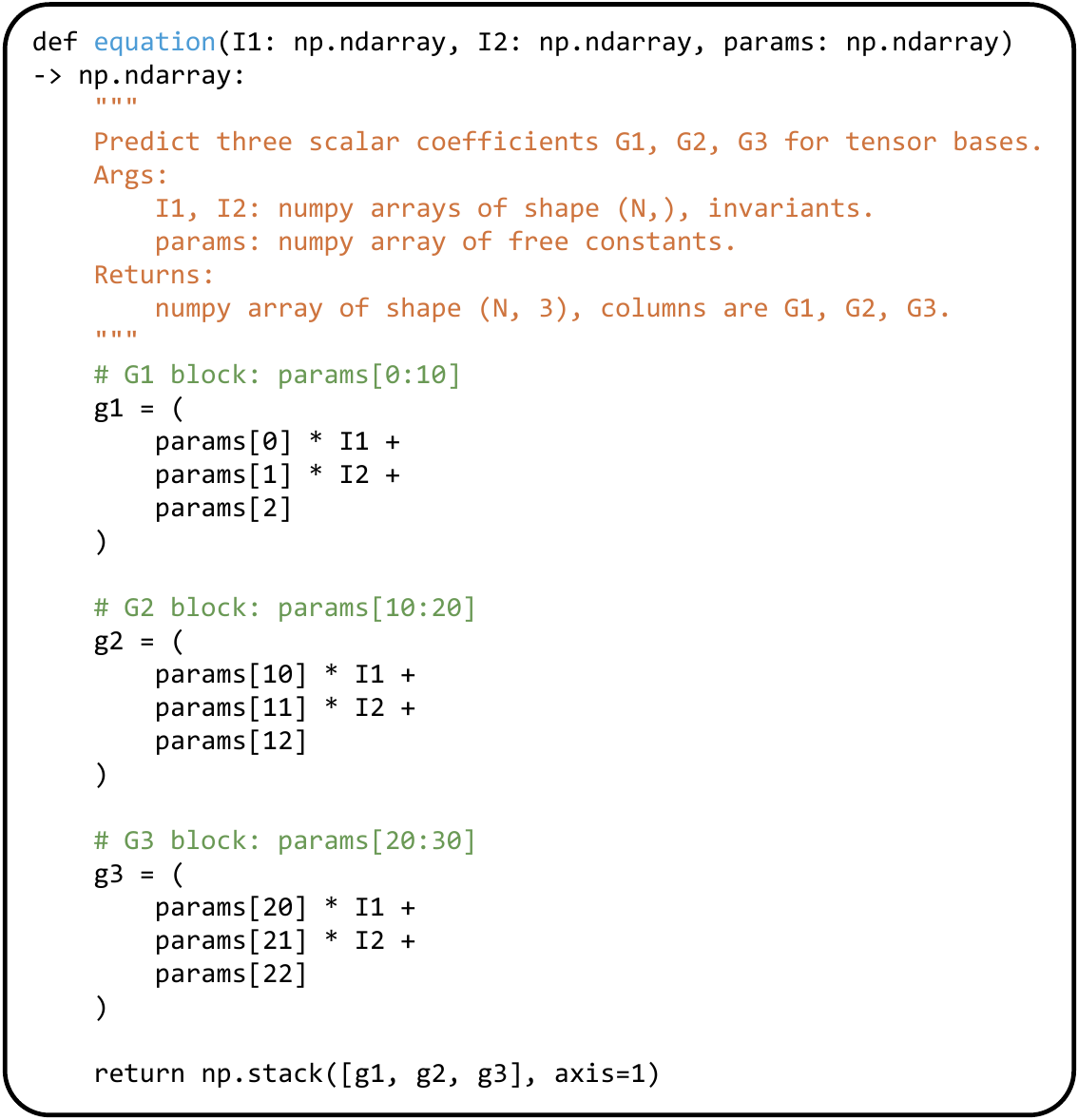}%
  }
  \caption{Equation Program Example.}
  \label{fig:spec3}
\end{figure}

\subsection{Back to Turbulence Modeling}
\label{subsec:turb_back_to_modeling}

\paragraph{From discovered expressions to a deployable closure}
After obtaining an explicit expression for the Reynolds-stress-related term from symbolic regression,
we rewrite the discovered formula into an OpenFOAM-readable form.

\paragraph{Embedding into OpenFOAM}
We embed the resulting closure into the existing \texttt{kOmegaSST} implementation by replacing the routine that evaluates the Reynolds stress
(or its modeled contribution) with the discovered expression, thereby forming an optimized turbulence model within the RANS framework.

\paragraph{Verification against DNS}
We then run RANS simulations with the modified model under the same configuration as the baseline and compare the resulting flow statistics
against the DNS reference data.
This completes the end-to-end turbulence modeling and validation pipeline considered in this work.

\subsection{Domain-Specific Constraints for Turbulence Modeling}
\label{subsec:turb_domain_constraints}

To encode expert knowledge as inductive biases, we augment the reward with a domain-specific penalty term:
\begin{equation}
P_{\text{domain}} \;=\; \sum_{j=1}^{4} w_j\,P_{\text{domain}}^{(j)},
\end{equation}
where each $P_{\text{domain}}^{(j)} \ge 0$ quantifies violations of the $j$-th turbulence constraint.

\paragraph{(1) Realizability.}
A physically admissible Reynolds stress tensor must be symmetric and positive semi-definite (PSD), which implies that all eigenvalues are nonnegative.
Non-realizable predictions are penalized by
\begin{equation}
P_{\text{domain}}^{(1)} \;=\; \frac{1}{N}\sum_{n=1}^{N} \max\Bigl(0,\,-\lambda_{\min}(\tau(x_n))\Bigr),
\label{eq:realizability_penalty}
\end{equation}
where $\lambda_{\min}(\cdot)$ denotes the smallest eigenvalue. This PSD requirement is a standard realizability condition for Reynolds stresses~\cite{Pope2000}. 

\paragraph{(2) Boundary-condition consistency.}
At a no-slip wall, the Reynolds stress tensor must decay to zero. Accordingly, non-vanishing stresses in a near-wall band
$\mathcal{W}=\{x: y(x)<y_0\}$ are penalized as
\begin{equation}
P_{\text{domain}}^{(2)} \;=\; \frac{1}{|\mathcal{W}|}\sum_{x_n\in\mathcal{W}} \|\tau(x_n)\|_F,
\label{eq:wall_penalty}
\end{equation}
where $y(x)$ is the wall distance field. In practice, $y(x)$ can be computed and exported using standard CFD post-processing tools
(OpenFOAM wall-distance utilities such as \texttt{wallDist} and \texttt{checkMesh-writeAllFields})~\cite{Monkewitz2021}. 

\paragraph{(3) Asymptotic scaling in the viscous sublayer.}
Near-wall Taylor expansions under the no-slip constraint imply that $u'=O(y)$ and, for incompressible flow, $v'=O(y^2)$. Consequently, the Reynolds shear stress
$-\overline{u'v'}$ exhibits cubic leading-order scaling, $O(y^3)$, in the viscous sublayer~\cite{Tennekes1972}. 
This constraint is enforced by evaluating the slope of $\log|\tau_{xy}^+|$ with respect to $\log(y^+)$ over the viscous sublayer range
$\mathcal{V}=\{x: y^+(x)\in[y_1^+,y_2^+]\}$:
\begin{equation}
p \;=\; \mathrm{slope}\bigl(\log|\tau_{xy}^+| \;\text{vs}\; \log(y^+)\bigr), \qquad
P_{\text{domain}}^{(3)} \;=\; |p-3|.
\label{eq:asymptotic_penalty}
\end{equation}
Here $y^+=y\,u_\tau/\nu$ is the wall unit, with friction velocity $u_\tau=\sqrt{\tau_w/\rho}$.
In CFD, $y^+$ and $\tau_w$ can be obtained through standard post-processing utilities (OpenFOAM \texttt{yPlus} and \texttt{wallShearStress}). 

\paragraph{(4) Energy consistency.}
Energy consistency is enforced by constraining the turbulent kinetic energy production implied by the predicted stress:
\begin{equation}
P_k(x) \;=\; -\tau_{ij}(x)\,\frac{\partial U_i}{\partial x_j}(x),
\label{eq:pk_def}
\end{equation}
and penalizing mismatches with the reference production $P_k^{\text{ref}}$ (computed from DNS-aligned fields) via
\begin{equation}
P_{\text{domain}}^{(4)} \;=\; \mathrm{MSE}\bigl(P_k,\,P_k^{\text{ref}}\bigr).
\label{eq:energy_penalty}
\end{equation}
The definition in Eq.~\eqref{eq:pk_def} is the standard production term in the TKE budget~\cite{Pope2000}. 

\paragraph{Efficient evaluation via elite-CFD refinement (multi-fidelity constraint scheduling).}
Constraints (2)--(4) require additional CFD-derived fields (wall distance $y$, wall shear stress $\tau_w$ and the associated $y^+$, and mean velocity gradients $\nabla U$),
which can be expensive to obtain at each candidate-evaluation step. To keep the symbolic regression loop efficient, we adopt a two-stage evaluation schedule.

\textbf{Stage-A (inexpensive screening, applied to all candidates)} evaluates the data-fitting score (MSE) together with realizability $P_{\text{domain}}^{(1)}$
using only the predicted stress tensor. 

\textbf{Stage-B (expensive physics checks, triggered by an improving best MSE)} is activated when a candidate in the current batch achieves a new best MSE value relative to all previously evaluated candidates. In this case, the candidate is treated as an elite solution, and CFD-based post-processing is performed to obtain the required auxiliary fields, compute $P_{\text{domain}}^{(2)}$--$P_{\text{domain}}^{(4)}$, and conduct an additional online fine-tuning step using the full domain-specific penalty. All non-elite candidates are trained only with the realizability constraint, whereas elite candidates receive the full turbulence constraint suite after CFD evaluation. This design enables rigorous physics enforcement without incurring prohibitive CFD costs during exploration.


\begin{thebibliography}{53}

%%% the Web address.

\ifx \showCODEN    \undefined \def \showCODEN     #1{\unskip}     \fi
\ifx \showISBNx    \undefined \def \showISBNx     #1{\unskip}     \fi
\ifx \showISBNxiii \undefined \def \showISBNxiii  #1{\unskip}     \fi
\ifx \showISSN     \undefined \def \showISSN      #1{\unskip}     \fi
\ifx \showLCCN     \undefined \def \showLCCN      #1{\unskip}     \fi
\ifx \shownote     \undefined \def \shownote      #1{#1}          \fi
\ifx \showarticletitle \undefined \def \showarticletitle #1{#1}   \fi
\ifx \showURL      \undefined \def \showURL       {\relax}        \fi
% invisible to TeX
\providecommand\bibfield[2]{#2}
\providecommand\bibinfo[2]{#2}
\providecommand\natexlab[1]{#1}
\providecommand\showeprint[2][]{arXiv:#2}

\bibitem[Aakash et~al\mbox{.}(2019)]%
        {Aakash2019}
\bibfield{author}{\bibinfo{person}{B.~S. Aakash}, \bibinfo{person}{JohnPatrick
  Connors}, {and} \bibinfo{person}{Michael~D Shields}.}
  \bibinfo{year}{2019}\natexlab{}.
\newblock \showarticletitle{Stress-strain data for aluminum 6061-T651 from 9
  lots at 6 temperatures under uniaxial and plane strain tension}.
\newblock \bibinfo{journal}{\emph{Data in Brief}}  \bibinfo{volume}{25}
  (\bibinfo{date}{Aug} \bibinfo{year}{2019}), \bibinfo{pages}{104085}.
\newblock
\showISSN{2352-3409}
\href{https://doi.org/10.1016/j.dib.2019.104085}{doi:\nolinkurl{10.1016/j.dib.2019.104085}}


\bibitem[Biggio et~al\mbox{.}(2021)]%
        {biggio2021neuralsymbolicregressionscales}
\bibfield{author}{\bibinfo{person}{Luca Biggio}, \bibinfo{person}{Tommaso
  Bendinelli}, \bibinfo{person}{Alexander Neitz}, \bibinfo{person}{Aurelien
  Lucchi}, {and} \bibinfo{person}{Giambattista Parascandolo}.}
  \bibinfo{year}{2021}\natexlab{}.
\newblock \bibinfo{title}{Neural Symbolic Regression that Scales}.
\newblock
\showeprint[arxiv]{2106.06427}~[cs.LG]
\urldef\tempurl%
\url{https://arxiv.org/abs/2106.06427}
\showURL{%
\tempurl}


\bibitem[Biggio* et~al\mbox{.}(2021)]%
        {Biggioetal21}
\bibfield{author}{\bibinfo{person}{L. Biggio*}, \bibinfo{person}{T.
  Bendinelli*}, \bibinfo{person}{A. Neitz}, \bibinfo{person}{A. Lucchi}, {and}
  \bibinfo{person}{G. Parascandolo}.} \bibinfo{year}{2021}\natexlab{}.
\newblock \showarticletitle{Neural Symbolic Regression that Scales}. In
  \bibinfo{booktitle}{\emph{Proceedings of 38th International Conference on
  Machine Learning (ICML 2021)}} \emph{(\bibinfo{series}{Proceedings of Machine
  Learning Research}, Vol.~\bibinfo{volume}{139})}. \bibinfo{publisher}{PMLR},
  \bibinfo{pages}{936--945}.
\newblock
\urldef\tempurl%
\url{https://proceedings.mlr.press/v139/biggio21a.html}
\showURL{%
\tempurl}
\newblock
\shownote{*equal contribution}.


\bibitem[Chen et~al\mbox{.}(2025)]%
        {ICLR2025_a76b693f}
\bibfield{author}{\bibinfo{person}{Jindou Chen}, \bibinfo{person}{Jidong Tian},
  \bibinfo{person}{Liang Wu}, \bibinfo{person}{ChenXinWei},
  \bibinfo{person}{Xiaokang Yang}, \bibinfo{person}{Yaohui Jin}, {and}
  \bibinfo{person}{Yanyan Xu}.} \bibinfo{year}{2025}\natexlab{}.
\newblock \showarticletitle{KinFormer: Generalizable Dynamical Symbolic
  Regression for Catalytic Organic Reaction Kinetics}. In
  \bibinfo{booktitle}{\emph{International Conference on Representation
  Learning}}, \bibfield{editor}{\bibinfo{person}{Y.~Yue},
  \bibinfo{person}{A.~Garg}, \bibinfo{person}{N.~Peng},
  \bibinfo{person}{F.~Sha}, {and} \bibinfo{person}{R.~Yu}} (Eds.),
  Vol.~\bibinfo{volume}{2025}. \bibinfo{pages}{67058--67080}.
\newblock
\urldef\tempurl%
\url{https://proceedings.iclr.cc/paper_files/paper/2025/file/a76b693f36916a5ed84d6e5b39a0dc03-Paper-Conference.pdf}
\showURL{%
\tempurl}


\bibitem[Cranmer(2023)]%
        {cranmer2023interpretablemachinelearningscience}
\bibfield{author}{\bibinfo{person}{Miles Cranmer}.}
  \bibinfo{year}{2023}\natexlab{}.
\newblock \bibinfo{title}{Interpretable Machine Learning for Science with PySR
  and SymbolicRegression.jl}.
\newblock
\showeprint[arxiv]{2305.01582}~[astro-ph.IM]
\urldef\tempurl%
\url{https://arxiv.org/abs/2305.01582}
\showURL{%
\tempurl}


\bibitem[Crochepierre et~al\mbox{.}(2022)]%
        {crochepierre:hal-03695471}
\bibfield{author}{\bibinfo{person}{Laure Crochepierre}, \bibinfo{person}{Lydia
  Boudjeloud-Assala}, {and} \bibinfo{person}{Vincent Barbesant}.}
  \bibinfo{year}{2022}\natexlab{}.
\newblock \showarticletitle{{Interactive Reinforcement Learning for Symbolic
  Regression from Multi-Format Human-Preference Feedbacks}}. In
  \bibinfo{booktitle}{\emph{{IJCAI 2022- 31st International Joint Conference on
  Artificial Intelligence}}}. \bibinfo{address}{Vienne, Austria}.
\newblock
\urldef\tempurl%
\url{https://hal.science/hal-03695471}
\showURL{%
\tempurl}


\bibitem[Deng et~al\mbox{.}(2023)]%
        {DENG2023109010}
\bibfield{author}{\bibinfo{person}{Song Deng}, \bibinfo{person}{Junjie Wang},
  \bibinfo{person}{Li Tao}, \bibinfo{person}{Su Zhang}, {and}
  \bibinfo{person}{Hongwei Sun}.} \bibinfo{year}{2023}\natexlab{}.
\newblock \showarticletitle{EV charging load forecasting model mining algorithm
  based on hybrid intelligence}.
\newblock \bibinfo{journal}{\emph{Computers and Electrical Engineering}}
  \bibinfo{volume}{112} (\bibinfo{year}{2023}), \bibinfo{pages}{109010}.
\newblock
\showISSN{0045-7906}
\href{https://doi.org/10.1016/j.compeleceng.2023.109010}{doi:\nolinkurl{10.1016/j.compeleceng.2023.109010}}


\bibitem[Du et~al\mbox{.}(2023)]%
        {du2023discoverdeepidentificationsymbolically}
\bibfield{author}{\bibinfo{person}{Mengge Du}, \bibinfo{person}{Yuntian Chen},
  {and} \bibinfo{person}{Dongxiao Zhang}.} \bibinfo{year}{2023}\natexlab{}.
\newblock \bibinfo{title}{DISCOVER: Deep identification of symbolically concise
  open-form PDEs via enhanced reinforcement-learning}.
\newblock
\showeprint[arxiv]{2210.02181}~[cs.LG]
\urldef\tempurl%
\url{https://arxiv.org/abs/2210.02181}
\showURL{%
\tempurl}


\bibitem[Fletcher(1987)]%
        {fletcher1987_practical_methods}
\bibfield{author}{\bibinfo{person}{Roger Fletcher}.}
  \bibinfo{year}{1987}\natexlab{}.
\newblock \bibinfo{booktitle}{\emph{Practical Methods of Optimization}
  (\bibinfo{edition}{2nd} ed.)}.
\newblock \bibinfo{publisher}{John Wiley \& Sons},
  \bibinfo{address}{Chichester, New York}.
\newblock
\showISBNx{0471915475}


\bibitem[Grayeli et~al\mbox{.}(2024)]%
        {grayeli2024symbolicregressionlearnedconcept}
\bibfield{author}{\bibinfo{person}{Arya Grayeli}, \bibinfo{person}{Atharva
  Sehgal}, \bibinfo{person}{Omar Costilla-Reyes}, \bibinfo{person}{Miles
  Cranmer}, {and} \bibinfo{person}{Swarat Chaudhuri}.}
  \bibinfo{year}{2024}\natexlab{}.
\newblock \bibinfo{title}{Symbolic Regression with a Learned Concept Library}.
\newblock
\showeprint[arxiv]{2409.09359}~[cs.LG]
\urldef\tempurl%
\url{https://arxiv.org/abs/2409.09359}
\showURL{%
\tempurl}


\bibitem[Guo et~al\mbox{.}(2025)]%
        {guo2025coevocontinualevolutionsymbolic}
\bibfield{author}{\bibinfo{person}{Ping Guo}, \bibinfo{person}{Qingfu Zhang},
  {and} \bibinfo{person}{Xi Lin}.} \bibinfo{year}{2025}\natexlab{}.
\newblock \bibinfo{title}{CoEvo: Continual Evolution of Symbolic Solutions
  Using Large Language Models}.
\newblock
\showeprint[arxiv]{2412.18890}~[cs.AI]
\urldef\tempurl%
\url{https://arxiv.org/abs/2412.18890}
\showURL{%
\tempurl}


\bibitem[Hu et~al\mbox{.}(2021)]%
        {hu2021loralowrankadaptationlarge}
\bibfield{author}{\bibinfo{person}{Edward~J. Hu}, \bibinfo{person}{Yelong
  Shen}, \bibinfo{person}{Phillip Wallis}, \bibinfo{person}{Zeyuan Allen-Zhu},
  \bibinfo{person}{Yuanzhi Li}, \bibinfo{person}{Shean Wang},
  \bibinfo{person}{Lu Wang}, {and} \bibinfo{person}{Weizhu Chen}.}
  \bibinfo{year}{2021}\natexlab{}.
\newblock \bibinfo{title}{LoRA: Low-Rank Adaptation of Large Language Models}.
\newblock
\showeprint[arxiv]{2106.09685}~[cs.CL]
\urldef\tempurl%
\url{https://arxiv.org/abs/2106.09685}
\showURL{%
\tempurl}


\bibitem[Huang et~al\mbox{.}(2025)]%
        {huang2025calmcoevolutionalgorithmslanguage}
\bibfield{author}{\bibinfo{person}{Ziyao Huang}, \bibinfo{person}{Weiwei Wu},
  \bibinfo{person}{Kui Wu}, \bibinfo{person}{Jianping Wang}, {and}
  \bibinfo{person}{Wei-Bin Lee}.} \bibinfo{year}{2025}\natexlab{}.
\newblock \bibinfo{title}{CALM: Co-evolution of Algorithms and Language Model
  for Automatic Heuristic Design}.
\newblock
\showeprint[arxiv]{2505.12285}~[cs.NE]
\urldef\tempurl%
\url{https://arxiv.org/abs/2505.12285}
\showURL{%
\tempurl}


\bibitem[Kamienny et~al\mbox{.}(2022)]%
        {kamienny2022endtoendsymbolicregressiontransformers}
\bibfield{author}{\bibinfo{person}{Pierre-Alexandre Kamienny},
  \bibinfo{person}{Stéphane d'Ascoli}, \bibinfo{person}{Guillaume Lample},
  {and} \bibinfo{person}{François Charton}.} \bibinfo{year}{2022}\natexlab{}.
\newblock \bibinfo{title}{End-to-end symbolic regression with transformers}.
\newblock
\showeprint[arxiv]{2204.10532}~[cs.LG]
\urldef\tempurl%
\url{https://arxiv.org/abs/2204.10532}
\showURL{%
\tempurl}


\bibitem[Kassianik et~al\mbox{.}(2025)]%
        {kassianik2025llama31foundationaisecurityllmbase8btechnicalreport}
\bibfield{author}{\bibinfo{person}{Paul Kassianik}, \bibinfo{person}{Baturay
  Saglam}, \bibinfo{person}{Alexander Chen}, \bibinfo{person}{Blaine Nelson},
  \bibinfo{person}{Anu Vellore}, \bibinfo{person}{Massimo Aufiero},
  \bibinfo{person}{Fraser Burch}, \bibinfo{person}{Dhruv Kedia},
  \bibinfo{person}{Avi Zohary}, \bibinfo{person}{Sajana Weerawardhena},
  \bibinfo{person}{Aman Priyanshu}, \bibinfo{person}{Adam Swanda},
  \bibinfo{person}{Amy Chang}, \bibinfo{person}{Hyrum Anderson},
  \bibinfo{person}{Kojin Oshiba}, \bibinfo{person}{Omar Santos},
  \bibinfo{person}{Yaron Singer}, {and} \bibinfo{person}{Amin Karbasi}.}
  \bibinfo{year}{2025}\natexlab{}.
\newblock \bibinfo{title}{Llama-3.1-FoundationAI-SecurityLLM-Base-8B Technical
  Report}.
\newblock
\showeprint[arxiv]{2504.21039}~[cs.CR]
\urldef\tempurl%
\url{https://arxiv.org/abs/2504.21039}
\showURL{%
\tempurl}


\bibitem[Koza(1990)]%
        {130444}
\bibfield{author}{\bibinfo{person}{J.R. Koza}.}
  \bibinfo{year}{1990}\natexlab{}.
\newblock \showarticletitle{Genetically breeding populations of computer
  programs to solve problems in artificial intelligence}. In
  \bibinfo{booktitle}{\emph{[1990] Proceedings of the 2nd International IEEE
  Conference on Tools for Artificial Intelligence}}. \bibinfo{pages}{819--827}.
\newblock
\href{https://doi.org/10.1109/TAI.1990.130444}{doi:\nolinkurl{10.1109/TAI.1990.130444}}


\bibitem[Landajuela et~al\mbox{.}(2022)]%
        {NEURIPS2022_dbca58f3}
\bibfield{author}{\bibinfo{person}{Mikel Landajuela},
  \bibinfo{person}{Chak~Shing Lee}, \bibinfo{person}{Jiachen Yang},
  \bibinfo{person}{Ruben Glatt}, \bibinfo{person}{Claudio~P Santiago},
  \bibinfo{person}{Ignacio Aravena}, \bibinfo{person}{Terrell Mundhenk},
  \bibinfo{person}{Garrett Mulcahy}, {and} \bibinfo{person}{Brenden~K
  Petersen}.} \bibinfo{year}{2022}\natexlab{}.
\newblock \showarticletitle{A Unified Framework for Deep Symbolic Regression}.
  In \bibinfo{booktitle}{\emph{Advances in Neural Information Processing
  Systems}}, \bibfield{editor}{\bibinfo{person}{S.~Koyejo},
  \bibinfo{person}{S.~Mohamed}, \bibinfo{person}{A.~Agarwal},
  \bibinfo{person}{D.~Belgrave}, \bibinfo{person}{K.~Cho}, {and}
  \bibinfo{person}{A.~Oh}} (Eds.), Vol.~\bibinfo{volume}{35}.
  \bibinfo{publisher}{Curran Associates, Inc.}, \bibinfo{pages}{33985--33998}.
\newblock
\urldef\tempurl%
\url{https://proceedings.neurips.cc/paper_files/paper/2022/file/dbca58f35bddc6e4003b2dd80e42f838-Paper-Conference.pdf}
\showURL{%
\tempurl}


\bibitem[Landajuela et~al\mbox{.}(2021)]%
        {landajuela2021improvingexplorationpolicygradient}
\bibfield{author}{\bibinfo{person}{Mikel Landajuela},
  \bibinfo{person}{Brenden~K. Petersen}, \bibinfo{person}{Soo~K. Kim},
  \bibinfo{person}{Claudio~P. Santiago}, \bibinfo{person}{Ruben Glatt},
  \bibinfo{person}{T.~Nathan Mundhenk}, \bibinfo{person}{Jacob~F. Pettit},
  {and} \bibinfo{person}{Daniel~M. Faissol}.} \bibinfo{year}{2021}\natexlab{}.
\newblock \bibinfo{title}{Improving exploration in policy gradient search:
  Application to symbolic optimization}.
\newblock
\showeprint[arxiv]{2107.09158}~[cs.LG]
\urldef\tempurl%
\url{https://arxiv.org/abs/2107.09158}
\showURL{%
\tempurl}


\bibitem[Li et~al\mbox{.}(2023)]%
        {Li2023TransformerbasedMF}
\bibfield{author}{\bibinfo{person}{Wenqiang Li}, \bibinfo{person}{Weijun Li},
  \bibinfo{person}{Linjun Sun}, \bibinfo{person}{Min Wu}, \bibinfo{person}{Lina
  Yu}, \bibinfo{person}{Jingyi Liu}, \bibinfo{person}{Yanjie Li}, {and}
  \bibinfo{person}{Song Tian}.} \bibinfo{year}{2023}\natexlab{}.
\newblock \showarticletitle{Transformer-based model for symbolic regression via
  joint supervised learning}. In \bibinfo{booktitle}{\emph{International
  Conference on Learning Representations}}.
\newblock
\urldef\tempurl%
\url{https://api.semanticscholar.org/CorpusID:259298765}
\showURL{%
\tempurl}


\bibitem[Ma et~al\mbox{.}(2024)]%
        {pmlr-v235-ma24m}
\bibfield{author}{\bibinfo{person}{Pingchuan Ma}, \bibinfo{person}{Tsun-Hsuan
  Wang}, \bibinfo{person}{Minghao Guo}, \bibinfo{person}{Zhiqing Sun},
  \bibinfo{person}{Joshua~B. Tenenbaum}, \bibinfo{person}{Daniela Rus},
  \bibinfo{person}{Chuang Gan}, {and} \bibinfo{person}{Wojciech Matusik}.}
  \bibinfo{year}{2024}\natexlab{}.
\newblock \showarticletitle{{LLM} and Simulation as Bilevel Optimizers: A New
  Paradigm to Advance Physical Scientific Discovery}. In
  \bibinfo{booktitle}{\emph{Proceedings of the 41st International Conference on
  Machine Learning}} \emph{(\bibinfo{series}{Proceedings of Machine Learning
  Research}, Vol.~\bibinfo{volume}{235})},
  \bibfield{editor}{\bibinfo{person}{Ruslan Salakhutdinov},
  \bibinfo{person}{Zico Kolter}, \bibinfo{person}{Katherine Heller},
  \bibinfo{person}{Adrian Weller}, \bibinfo{person}{Nuria Oliver},
  \bibinfo{person}{Jonathan Scarlett}, {and} \bibinfo{person}{Felix
  Berkenkamp}} (Eds.). \bibinfo{publisher}{PMLR},
  \bibinfo{pages}{33940--33962}.
\newblock
\urldef\tempurl%
\url{https://proceedings.mlr.press/v235/ma24m.html}
\showURL{%
\tempurl}


\bibitem[Makke and Chawla(2024a)]%
        {10.1093pnasnexuspgae467}
\bibfield{author}{\bibinfo{person}{Nour Makke} {and} \bibinfo{person}{Sanjay
  Chawla}.} \bibinfo{year}{2024}\natexlab{a}.
\newblock \showarticletitle{Data-driven discovery of Tsallis-like distribution
  using symbolic regression in high-energy physics}.
\newblock \bibinfo{journal}{\emph{PNAS Nexus}} \bibinfo{volume}{3},
  \bibinfo{number}{11} (\bibinfo{date}{10} \bibinfo{year}{2024}),
  \bibinfo{pages}{pgae467}.
\newblock
\showISSN{2752-6542}
\showeprint{https://academic.oup.com/pnasnexus/article-pdf/3/11/pgae467/60816181/pgae467.pdf}
\href{https://doi.org/10.1093/pnasnexus/pgae467}{doi:\nolinkurl{10.1093/pnasnexus/pgae467}}


\bibitem[Makke and Chawla(2024b)]%
        {article}
\bibfield{author}{\bibinfo{person}{Nour Makke} {and} \bibinfo{person}{Sanjay
  Chawla}.} \bibinfo{year}{2024}\natexlab{b}.
\newblock \showarticletitle{Interpretable scientific discovery with symbolic
  regression: a review}.
\newblock \bibinfo{journal}{\emph{Artificial Intelligence Review}}
  \bibinfo{volume}{57} (\bibinfo{date}{01} \bibinfo{year}{2024}).
\newblock
\href{https://doi.org/10.1007/s10462-023-10622-0}{doi:\nolinkurl{10.1007/s10462-023-10622-0}}


\bibitem[Menter(1994)]%
        {menter1994two}
\bibfield{author}{\bibinfo{person}{Florian~R Menter}.}
  \bibinfo{year}{1994}\natexlab{}.
\newblock \showarticletitle{Two-equation eddy-viscosity turbulence models for
  engineering applications}.
\newblock \bibinfo{journal}{\emph{AIAA journal}} \bibinfo{volume}{32},
  \bibinfo{number}{8} (\bibinfo{year}{1994}), \bibinfo{pages}{1598--1605}.
\newblock


\bibitem[Menter et~al\mbox{.}(2003)]%
        {menter2003ten}
\bibfield{author}{\bibinfo{person}{Florian~R Menter}, \bibinfo{person}{Martin
  Kuntz}, \bibinfo{person}{Robin Langtry}, {et~al\mbox{.}}}
  \bibinfo{year}{2003}\natexlab{}.
\newblock \showarticletitle{Ten years of industrial experience with the SST
  turbulence model}.
\newblock \bibinfo{journal}{\emph{Turbulence, heat and mass transfer}}
  \bibinfo{volume}{4}, \bibinfo{number}{1} (\bibinfo{year}{2003}),
  \bibinfo{pages}{625--632}.
\newblock


\bibitem[Merler et~al\mbox{.}(2024)]%
        {Merler_2024}
\bibfield{author}{\bibinfo{person}{Matteo Merler}, \bibinfo{person}{Katsiaryna
  Haitsiukevich}, \bibinfo{person}{Nicola Dainese}, {and}
  \bibinfo{person}{Pekka Marttinen}.} \bibinfo{year}{2024}\natexlab{}.
\newblock \showarticletitle{In-Context Symbolic Regression: Leveraging Large
  Language Models for Function Discovery}. In
  \bibinfo{booktitle}{\emph{Proceedings of the 62nd Annual Meeting of the
  Association for Computational Linguistics (Volume 4: Student Research
  Workshop)}}. \bibinfo{publisher}{Association for Computational Linguistics},
  \bibinfo{pages}{589–606}.
\newblock
\href{https://doi.org/10.18653/v1/2024.acl-srw.49}{doi:\nolinkurl{10.18653/v1/2024.acl-srw.49}}


\bibitem[MOCHIZUKI and OSAKA(2000)]%
        {MOCHIZUKI2000}
\bibfield{author}{\bibinfo{person}{Shinsuke MOCHIZUKI} {and}
  \bibinfo{person}{Hideo OSAKA}.} \bibinfo{year}{2000}\natexlab{}.
\newblock \showarticletitle{Management of a Stronger Wall Jet by a Pair of
  Streamwise Vortices. Reynolds Stress Tensor and Production Tenns.}
\newblock \bibinfo{journal}{\emph{TRANSACTIONS OF THE JAPAN SOCIETY OF
  MECHANICAL ENGINEERS Series B}} \bibinfo{volume}{66}, \bibinfo{number}{646}
  (\bibinfo{year}{2000}), \bibinfo{pages}{1309–1317}.
\newblock
\showISSN{1884-8346}
\href{https://doi.org/10.1299/kikaib.66.646_1309}{doi:\nolinkurl{10.1299/kikaib.66.646_1309}}


\bibitem[Monkewitz(2021)]%
        {Monkewitz2021}
\bibfield{author}{\bibinfo{person}{Peter~A. Monkewitz}.}
  \bibinfo{year}{2021}\natexlab{}.
\newblock \showarticletitle{Asymptotics of streamwise Reynolds stress in wall
  turbulence}.
\newblock \bibinfo{journal}{\emph{Journal of Fluid Mechanics}}
  \bibinfo{volume}{931} (\bibinfo{date}{Nov.} \bibinfo{year}{2021}).
\newblock
\showISSN{1469-7645}
\href{https://doi.org/10.1017/jfm.2021.924}{doi:\nolinkurl{10.1017/jfm.2021.924}}


\bibitem[Monod(1949)]%
        {annurev:contentjournals10.1146annurev.mi.03.100149.002103}
\bibfield{author}{\bibinfo{person}{Jacques Monod}.}
  \bibinfo{year}{1949}\natexlab{}.
\newblock \showarticletitle{THE GROWTH OF BACTERIAL CULTURES}.
\newblock \bibinfo{journal}{\emph{Annual Review of Microbiology}}
  \bibinfo{volume}{3}, \bibinfo{number}{Volume 3, 1949} (\bibinfo{year}{1949}),
  \bibinfo{pages}{371--394}.
\newblock
\showISSN{1545-3251}
\href{https://doi.org/10.1146/annurev.mi.03.100149.002103}{doi:\nolinkurl{10.1146/annurev.mi.03.100149.002103}}


\bibitem[Mundhenk et~al\mbox{.}(2021)]%
        {mundhenk2021symbolicregressionneuralguidedgenetic}
\bibfield{author}{\bibinfo{person}{T.~Nathan Mundhenk}, \bibinfo{person}{Mikel
  Landajuela}, \bibinfo{person}{Ruben Glatt}, \bibinfo{person}{Claudio~P.
  Santiago}, \bibinfo{person}{Daniel~M. Faissol}, {and}
  \bibinfo{person}{Brenden~K. Petersen}.} \bibinfo{year}{2021}\natexlab{}.
\newblock \bibinfo{title}{Symbolic Regression via Neural-Guided Genetic
  Programming Population Seeding}.
\newblock
\showeprint[arxiv]{2111.00053}~[cs.NE]
\urldef\tempurl%
\url{https://arxiv.org/abs/2111.00053}
\showURL{%
\tempurl}


\bibitem[Neamtiu et~al\mbox{.}(2005)]%
        {10.11451083142.1083143}
\bibfield{author}{\bibinfo{person}{Iulian Neamtiu}, \bibinfo{person}{Jeffrey~S.
  Foster}, {and} \bibinfo{person}{Michael Hicks}.}
  \bibinfo{year}{2005}\natexlab{}.
\newblock \showarticletitle{Understanding source code evolution using abstract
  syntax tree matching}. In \bibinfo{booktitle}{\emph{Proceedings of the 2005
  International Workshop on Mining Software Repositories}} (St. Louis,
  Missouri) \emph{(\bibinfo{series}{MSR '05})}. \bibinfo{publisher}{Association
  for Computing Machinery}, \bibinfo{address}{New York, NY, USA},
  \bibinfo{pages}{1–5}.
\newblock
\showISBNx{1595931236}
\href{https://doi.org/10.1145/1083142.1083143}{doi:\nolinkurl{10.1145/1083142.1083143}}


\bibitem[Petersen et~al\mbox{.}(2021)]%
        {petersen2021deepsymbolicregressionrecovering}
\bibfield{author}{\bibinfo{person}{Brenden~K. Petersen}, \bibinfo{person}{Mikel
  Landajuela}, \bibinfo{person}{T.~Nathan Mundhenk},
  \bibinfo{person}{Claudio~P. Santiago}, \bibinfo{person}{Soo~K. Kim}, {and}
  \bibinfo{person}{Joanne~T. Kim}.} \bibinfo{year}{2021}\natexlab{}.
\newblock \bibinfo{title}{Deep symbolic regression: Recovering mathematical
  expressions from data via risk-seeking policy gradients}.
\newblock
\showeprint[arxiv]{1912.04871}~[cs.LG]
\urldef\tempurl%
\url{https://arxiv.org/abs/1912.04871}
\showURL{%
\tempurl}


\bibitem[Pope(2000)]%
        {Pope2000}
\bibfield{author}{\bibinfo{person}{Stephen~B. Pope}.}
  \bibinfo{year}{2000}\natexlab{}.
\newblock \bibinfo{booktitle}{\emph{Turbulent Flows}}.
\newblock \bibinfo{publisher}{Cambridge University Press}.
\newblock
\showISBNx{9780511840531}
\href{https://doi.org/10.1017/cbo9780511840531}{doi:\nolinkurl{10.1017/cbo9780511840531}}


\bibitem[Pourcel et~al\mbox{.}(2025)]%
        {pourcel2025selfimprovinglanguagemodelsevolutionary}
\bibfield{author}{\bibinfo{person}{Julien Pourcel}, \bibinfo{person}{Cédric
  Colas}, {and} \bibinfo{person}{Pierre-Yves Oudeyer}.}
  \bibinfo{year}{2025}\natexlab{}.
\newblock \bibinfo{title}{Self-Improving Language Models for Evolutionary
  Program Synthesis: A Case Study on ARC-AGI}.
\newblock
\showeprint[arxiv]{2507.14172}~[cs.LG]
\urldef\tempurl%
\url{https://arxiv.org/abs/2507.14172}
\showURL{%
\tempurl}


\bibitem[Reuter et~al\mbox{.}(2023)]%
        {10.1007978-3-031-29573-7_3}
\bibfield{author}{\bibinfo{person}{Julia Reuter}, \bibinfo{person}{Hani
  Elmestikawy}, \bibinfo{person}{Fabien Evrard}, \bibinfo{person}{Sanaz
  Mostaghim}, {and} \bibinfo{person}{Berend van Wachem}.}
  \bibinfo{year}{2023}\natexlab{}.
\newblock \showarticletitle{Graph Networks as Inductive Bias for Genetic
  Programming: Symbolic Models for Particle-Laden Flows}. In
  \bibinfo{booktitle}{\emph{Genetic Programming}},
  \bibfield{editor}{\bibinfo{person}{Gisele Pappa}, \bibinfo{person}{Mario
  Giacobini}, {and} \bibinfo{person}{Zdenek Vasicek}} (Eds.).
  \bibinfo{publisher}{Springer Nature Switzerland}, \bibinfo{address}{Cham},
  \bibinfo{pages}{36--51}.
\newblock
\showISBNx{978-3-031-29573-7}


\bibitem[Romera-Paredes et~al\mbox{.}(2024)]%
        {romera2024mathematical}
\bibfield{author}{\bibinfo{person}{Bernardino Romera-Paredes},
  \bibinfo{person}{Mohammadamin Barekatain}, \bibinfo{person}{Alexander
  Novikov}, \bibinfo{person}{Matej Balog}, \bibinfo{person}{M~Pawan Kumar},
  \bibinfo{person}{Emilien Dupont}, \bibinfo{person}{Francisco~JR Ruiz},
  \bibinfo{person}{Jordan~S Ellenberg}, \bibinfo{person}{Pengming Wang},
  \bibinfo{person}{Omar Fawzi}, {et~al\mbox{.}}}
  \bibinfo{year}{2024}\natexlab{}.
\newblock \showarticletitle{Mathematical discoveries from program search with
  large language models}.
\newblock \bibinfo{journal}{\emph{Nature}} \bibinfo{volume}{625},
  \bibinfo{number}{7995} (\bibinfo{year}{2024}), \bibinfo{pages}{468--475}.
\newblock


\bibitem[Rosso et~al\mbox{.}(1995)]%
        {Rosso1995}
\bibfield{author}{\bibinfo{person}{L Rosso}, \bibinfo{person}{J.~R. Lobry},
  \bibinfo{person}{S Bajard}, {and} \bibinfo{person}{J.~P. Flandrois}.}
  \bibinfo{year}{1995}\natexlab{}.
\newblock \showarticletitle{Convenient Model To Describe the Combined Effects
  of Temperature and {pH} on Microbial Growth}.
\newblock \bibinfo{journal}{\emph{Applied and Environmental Microbiology}}
  \bibinfo{volume}{61}, \bibinfo{number}{2} (\bibinfo{date}{Feb}
  \bibinfo{year}{1995}), \bibinfo{pages}{610--6}.
\newblock
\showISSN{0099-2240}
\href{https://doi.org/10.1128/aem.61.2.610-616.1995}{doi:\nolinkurl{10.1128/aem.61.2.610-616.1995}}


\bibitem[Schmidt and Lipson(2009)]%
        {doi:10.1126science.1165893}
\bibfield{author}{\bibinfo{person}{Michael Schmidt} {and} \bibinfo{person}{Hod
  Lipson}.} \bibinfo{year}{2009}\natexlab{}.
\newblock \showarticletitle{Distilling Free-Form Natural Laws from Experimental
  Data}.
\newblock \bibinfo{journal}{\emph{Science}} \bibinfo{volume}{324},
  \bibinfo{number}{5923} (\bibinfo{year}{2009}), \bibinfo{pages}{81--85}.
\newblock
\showeprint{https://www.science.org/doi/pdf/10.1126/science.1165893}
\href{https://doi.org/10.1126/science.1165893}{doi:\nolinkurl{10.1126/science.1165893}}


\bibitem[Shao et~al\mbox{.}(2024)]%
        {shao2024deepseekmathpushinglimitsmathematical}
\bibfield{author}{\bibinfo{person}{Zhihong Shao}, \bibinfo{person}{Peiyi Wang},
  \bibinfo{person}{Qihao Zhu}, \bibinfo{person}{Runxin Xu},
  \bibinfo{person}{Junxiao Song}, \bibinfo{person}{Xiao Bi},
  \bibinfo{person}{Haowei Zhang}, \bibinfo{person}{Mingchuan Zhang},
  \bibinfo{person}{Y.~K. Li}, \bibinfo{person}{Y. Wu}, {and}
  \bibinfo{person}{Daya Guo}.} \bibinfo{year}{2024}\natexlab{}.
\newblock \bibinfo{title}{DeepSeekMath: Pushing the Limits of Mathematical
  Reasoning in Open Language Models}.
\newblock
\showeprint[arxiv]{2402.03300}~[cs.CL]
\urldef\tempurl%
\url{https://arxiv.org/abs/2402.03300}
\showURL{%
\tempurl}


\bibitem[Shi et~al\mbox{.}(2024)]%
        {shi2024alphaforgeframeworkdynamicallycombine}
\bibfield{author}{\bibinfo{person}{Hao Shi}, \bibinfo{person}{Weili Song},
  \bibinfo{person}{Xinting Zhang}, \bibinfo{person}{Jiahe Shi},
  \bibinfo{person}{Cuicui Luo}, \bibinfo{person}{Xiang Ao},
  \bibinfo{person}{Hamid Arian}, {and} \bibinfo{person}{Luis Seco}.}
  \bibinfo{year}{2024}\natexlab{}.
\newblock \bibinfo{title}{AlphaForge: A Framework to Mine and Dynamically
  Combine Formulaic Alpha Factors}.
\newblock
\showeprint[arxiv]{2406.18394}~[q-fin.CP]
\urldef\tempurl%
\url{https://arxiv.org/abs/2406.18394}
\showURL{%
\tempurl}


\bibitem[Shojaee et~al\mbox{.}(2025a)]%
        {shojaee2025llmsrscientificequationdiscovery}
\bibfield{author}{\bibinfo{person}{Parshin Shojaee}, \bibinfo{person}{Kazem
  Meidani}, \bibinfo{person}{Shashank Gupta}, \bibinfo{person}{Amir~Barati
  Farimani}, {and} \bibinfo{person}{Chandan~K Reddy}.}
  \bibinfo{year}{2025}\natexlab{a}.
\newblock \bibinfo{title}{LLM-SR: Scientific Equation Discovery via Programming
  with Large Language Models}.
\newblock
\showeprint[arxiv]{2404.18400}~[cs.LG]
\urldef\tempurl%
\url{https://arxiv.org/abs/2404.18400}
\showURL{%
\tempurl}


\bibitem[Shojaee et~al\mbox{.}(2025b)]%
        {shojaee2025llmsrbenchnewbenchmarkscientific}
\bibfield{author}{\bibinfo{person}{Parshin Shojaee}, \bibinfo{person}{Ngoc-Hieu
  Nguyen}, \bibinfo{person}{Kazem Meidani}, \bibinfo{person}{Amir~Barati
  Farimani}, \bibinfo{person}{Khoa~D Doan}, {and} \bibinfo{person}{Chandan~K
  Reddy}.} \bibinfo{year}{2025}\natexlab{b}.
\newblock \bibinfo{title}{LLM-SRBench: A New Benchmark for Scientific Equation
  Discovery with Large Language Models}.
\newblock
\showeprint[arxiv]{2504.10415}~[cs.CL]
\urldef\tempurl%
\url{https://arxiv.org/abs/2504.10415}
\showURL{%
\tempurl}


\bibitem[Tang et~al\mbox{.}(2023a)]%
        {doi:10.10635.0135638}
\bibfield{author}{\bibinfo{person}{Hongwei Tang}, \bibinfo{person}{Yan Wang},
  \bibinfo{person}{Tongguang Wang}, {and} \bibinfo{person}{Linlin Tian}.}
  \bibinfo{year}{2023}\natexlab{a}.
\newblock \showarticletitle{Discovering explicit Reynolds-averaged turbulence
  closures for turbulent separated flows through deep learning-based symbolic
  regression with non-linear corrections}.
\newblock \bibinfo{journal}{\emph{Physics of Fluids}} \bibinfo{volume}{35},
  \bibinfo{number}{2} (\bibinfo{year}{2023}), \bibinfo{pages}{025118}.
\newblock
\showeprint{https://doi.org/10.1063/5.0135638}
\href{https://doi.org/10.1063/5.0135638}{doi:\nolinkurl{10.1063/5.0135638}}


\bibitem[Tang et~al\mbox{.}(2023b)]%
        {Tang2023}
\bibfield{author}{\bibinfo{person}{Hongwei Tang}, \bibinfo{person}{Yan Wang},
  \bibinfo{person}{Tongguang Wang}, {and} \bibinfo{person}{Linlin Tian}.}
  \bibinfo{year}{2023}\natexlab{b}.
\newblock \showarticletitle{Discovering explicit Reynolds-averaged turbulence
  closures for turbulent separated flows through deep learning-based symbolic
  regression with non-linear corrections}.
\newblock \bibinfo{journal}{\emph{Physics of Fluids}} \bibinfo{volume}{35},
  \bibinfo{number}{2} (\bibinfo{date}{Feb.} \bibinfo{year}{2023}).
\newblock
\showISSN{1089-7666}
\href{https://doi.org/10.1063/5.0135638}{doi:\nolinkurl{10.1063/5.0135638}}


\bibitem[Tennekes and Lumley(1972)]%
        {Tennekes1972}
\bibfield{author}{\bibinfo{person}{Henk Tennekes} {and}
  \bibinfo{person}{John~L. Lumley}.} \bibinfo{year}{1972}\natexlab{}.
\newblock \bibinfo{booktitle}{\emph{A First Course in Turbulence}}.
\newblock \bibinfo{publisher}{The MIT Press}.
\newblock
\showISBNx{9780262310901}
\href{https://doi.org/10.7551/mitpress/3014.001.0001}{doi:\nolinkurl{10.7551/mitpress/3014.001.0001}}


\bibitem[Valipour et~al\mbox{.}(2021)]%
        {valipour2021symbolicgptgenerativetransformermodel}
\bibfield{author}{\bibinfo{person}{Mojtaba Valipour}, \bibinfo{person}{Bowen
  You}, \bibinfo{person}{Maysum Panju}, {and} \bibinfo{person}{Ali Ghodsi}.}
  \bibinfo{year}{2021}\natexlab{}.
\newblock \bibinfo{title}{SymbolicGPT: A Generative Transformer Model for
  Symbolic Regression}.
\newblock
\showeprint[arxiv]{2106.14131}~[cs.LG]
\urldef\tempurl%
\url{https://arxiv.org/abs/2106.14131}
\showURL{%
\tempurl}


\bibitem[Vastl et~al\mbox{.}(2024)]%
        {10462113}
\bibfield{author}{\bibinfo{person}{Martin Vastl}, \bibinfo{person}{Jonáš
  Kulhánek}, \bibinfo{person}{Jiří Kubalík}, \bibinfo{person}{Erik Derner},
  {and} \bibinfo{person}{Robert Babuška}.} \bibinfo{year}{2024}\natexlab{}.
\newblock \showarticletitle{SymFormer: End-to-End Symbolic Regression Using
  Transformer-Based Architecture}.
\newblock \bibinfo{journal}{\emph{IEEE Access}}  \bibinfo{volume}{12}
  (\bibinfo{year}{2024}), \bibinfo{pages}{37840--37849}.
\newblock
\href{https://doi.org/10.1109/ACCESS.2024.3374649}{doi:\nolinkurl{10.1109/ACCESS.2024.3374649}}


\bibitem[Virgolin and Pissis(2022)]%
        {virgolin2022symbolicregressionnphard}
\bibfield{author}{\bibinfo{person}{Marco Virgolin} {and}
  \bibinfo{person}{Solon~P. Pissis}.} \bibinfo{year}{2022}\natexlab{}.
\newblock \bibinfo{title}{Symbolic Regression is NP-hard}.
\newblock
\showeprint[arxiv]{2207.01018}~[cs.NE]
\urldef\tempurl%
\url{https://arxiv.org/abs/2207.01018}
\showURL{%
\tempurl}


\bibitem[Wahlquist et~al\mbox{.}(2024)]%
        {Wahlquist2024}
\bibfield{author}{\bibinfo{person}{Ylva Wahlquist}, \bibinfo{person}{Jesper
  Sundell}, {and} \bibinfo{person}{Kristian Soltesz}.}
  \bibinfo{year}{2024}\natexlab{}.
\newblock \showarticletitle{Learning pharmacometric covariate model structures
  with symbolic regression networks}.
\newblock \bibinfo{journal}{\emph{Journal of Pharmacokinetics and
  Pharmacodynamics}} \bibinfo{volume}{51}, \bibinfo{number}{2}
  (\bibinfo{year}{2024}), \bibinfo{pages}{155--167}.
\newblock
\showISSN{1573-8744}
\href{https://doi.org/10.1007/s10928-023-09887-3}{doi:\nolinkurl{10.1007/s10928-023-09887-3}}


\bibitem[WANG et~al\mbox{.}(2019)]%
        {0258-1825(2019)03-0419-07}
\bibfield{author}{\bibinfo{person}{Wei WANG}, \bibinfo{person}{Yang ZHANG},
  {and} \bibinfo{person}{Lili CHEN}.} \bibinfo{year}{2019}\natexlab{}.
\newblock \showarticletitle{An application of the shear stress transport
  low-Reynolds-number k-ε turbulence model on turbulent flows}.
\newblock \bibinfo{journal}{\emph{ACTA AERODYNAMICA SINICA}}
  \bibinfo{volume}{37}, \bibinfo{number}{3} (\bibinfo{year}{2019}),
  \bibinfo{pages}{419--425}.
\newblock
\showISSN{0258-1825}
\href{https://doi.org/10.7638/kqdlxxb-2016.0158}{doi:\nolinkurl{10.7638/kqdlxxb-2016.0158}}


\bibitem[Weller et~al\mbox{.}(1998)]%
        {OpenFOAM}
\bibfield{author}{\bibinfo{person}{H.G. Weller}, \bibinfo{person}{Gavin Tabor},
  \bibinfo{person}{Hrvoje Jasak}, {and} \bibinfo{person}{Christer Fureby}.}
  \bibinfo{year}{1998}\natexlab{}.
\newblock \showarticletitle{A Tensorial Approach to Computational Continuum
  Mechanics Using Object Orientated Techniques}.
\newblock \bibinfo{journal}{\emph{Computers in Physics}}  \bibinfo{volume}{12}
  (\bibinfo{date}{11} \bibinfo{year}{1998}), \bibinfo{pages}{620--631}.
\newblock
\href{https://doi.org/10.1063/1.168744}{doi:\nolinkurl{10.1063/1.168744}}


\bibitem[Xia et~al\mbox{.}(2025)]%
        {xia2025srscientistscientificequationdiscovery}
\bibfield{author}{\bibinfo{person}{Shijie Xia}, \bibinfo{person}{Yuhan Sun},
  {and} \bibinfo{person}{Pengfei Liu}.} \bibinfo{year}{2025}\natexlab{}.
\newblock \bibinfo{title}{SR-Scientist: Scientific Equation Discovery With
  Agentic AI}.
\newblock
\showeprint[arxiv]{2510.11661}~[cs.AI]
\urldef\tempurl%
\url{https://arxiv.org/abs/2510.11661}
\showURL{%
\tempurl}


\bibitem[Xiao et~al\mbox{.}(2020)]%
        {XIAO2020104431}
\bibfield{author}{\bibinfo{person}{Heng Xiao}, \bibinfo{person}{Jin-Long Wu},
  \bibinfo{person}{Sylvain Laizet}, {and} \bibinfo{person}{Lian Duan}.}
  \bibinfo{year}{2020}\natexlab{}.
\newblock \showarticletitle{Flows over periodic hills of parameterized
  geometries: A dataset for data-driven turbulence modeling from direct
  simulations}.
\newblock \bibinfo{journal}{\emph{Computers \& Fluids}}  \bibinfo{volume}{200}
  (\bibinfo{year}{2020}), \bibinfo{pages}{104431}.
\newblock
\showISSN{0045-7930}
\href{https://doi.org/10.1016/j.compfluid.2020.104431}{doi:\nolinkurl{10.1016/j.compfluid.2020.104431}}


\bibitem[Zhang et~al\mbox{.}(2025)]%
        {zhang2025ragsr}
\bibfield{author}{\bibinfo{person}{Hengzhe Zhang}, \bibinfo{person}{Qi Chen},
  \bibinfo{person}{Bing XUE}, \bibinfo{person}{Wolfgang Banzhaf}, {and}
  \bibinfo{person}{Mengjie Zhang}.} \bibinfo{year}{2025}\natexlab{}.
\newblock \showarticletitle{{RAG}-{SR}: Retrieval-Augmented Generation for
  Neural Symbolic Regression}. In \bibinfo{booktitle}{\emph{The Thirteenth
  International Conference on Learning Representations}}.
\newblock
\urldef\tempurl%
\url{https://openreview.net/forum?id=NdHka08uWn}
\showURL{%
\tempurl}


\end{thebibliography}
\end{document}